\documentclass{article}

\pdfoutput=1

\usepackage{arxiv}
\usepackage{times}  
\usepackage{helvet} 
\usepackage{courier}  
\usepackage[hyphens]{url}  
\usepackage{graphicx} 
\usepackage{graphicx}  
\usepackage{comment}
\usepackage{amsfonts}
\usepackage{mathtools}
\usepackage{amsthm}
\usepackage{times}
\usepackage{graphicx} 
\usepackage{subfigure} 
\usepackage{algorithm}

\usepackage{algorithmic}
\newtheorem{theorem}{Theorem}
\newtheorem{proposition}{Proposition}
\newtheorem{corollary}{Corollary} 
\newtheorem{lemma}{Lemma}
\newcommand{\argmin}{\operatornamewithlimits{argmin}}
\newcommand{\cS}{\mathcal{S}}
\newcommand{\cR}{\mathcal{R}}
\newcommand{\cH}{\mathcal{H}}
\newcommand{\bbR}{\mathbb{R}}
\newcommand{\Tr}{\mathrm{Tr}}
\newcommand{\bbE}{\mathbb{E}}
\usepackage{tikz}
\newcommand*\circled[1]{\tikz[baseline=(char.base)]{
            \node[shape=circle,draw,inner sep=2pt] (char) {#1};}}

\title{Trading-Off Static and Dynamic Regret in Online Least-Squares
and Beyond}

\author{
  Jianjun Yuan\\
  Department of Electrical and Computer Engineering\\
  University of Minnesota\\
  Minneapolis, MN, 55455 \\
  \texttt{yuanx270@umn.edu} \\
  \And
  Andrew Lamperski \\
  Department of Electrical and Computer Engineering\\
  University of Minnesota\\
  Minneapolis, MN, 55455 \\
  \texttt{alampers@umn.edu} \\
}

\begin{document}

\maketitle

\begin{abstract}
Recursive least-squares algorithms often use forgetting factors as a heuristic to
 adapt to non-stationary data streams.
 The first contribution of this paper rigorously characterizes the 
 effect of forgetting factors for a class of online Newton algorithms.
 For exp-concave and strongly convex objectives, the algorithms
 achieve the dynamic regret of $\max\{O(\log T),O(\sqrt{TV})\}$,
 where $V$ is a bound on the path length of the comparison sequence.
 In particular, we show how classic
 recursive least-squares with a forgetting factor achieves this
 dynamic regret bound.
 By varying $V$, we obtain a trade-off between
 static and dynamic regret. 
 In order to obtain more computationally efficient algorithms, our
 second contribution is a novel gradient descent step size rule for
 strongly convex functions.
 Our gradient descent rule recovers the order optimal dynamic regret bounds
 described above. 
 For smooth problems, we can also obtain static regret of $O(T^{1-\beta})$ and
 dynamic regret of $O(T^\beta V^*)$, where $\beta \in (0,1)$ and $V^*$
 is the path length of the sequence of minimizers.  
 By varying $\beta$, we obtain a trade-off between static and dynamic
 regret. 
 
\end{abstract}

\section{Introduction}

Online learning algorithms are designed to solve prediction and
learning problems for streaming data 
or batch data whose volume is too large to be processed all at once.
Applications include 
online auctions \cite{blum2004online}, online classification and regression \cite{crammer2006online},
online subspace identification \cite{yuan2019online},
as well as online resource allocation \cite{yuan2018online}.
%

The general procedure for online learning algorithms
is as follows:
at each time $t$, before the true time-dependent objective function $f_t(\theta)$ is revealed,
we need to make the prediction, $\theta_t$, 
based on the history of the observations $f_i(\theta)$, $i<t$.
Then the value of $f_t(\theta_t)$ is the loss suffered due to the lack
of the knowledge for the true objective function $f_t(\theta)$.  
Our prediction of $\theta$ is then updated to include the information of $f_t(\theta)$.
This whole process is repeated until termination. The functions,
$f_t(\theta)$,
can be chosen from a function class in an arbitrary, possibly
adversarial manner.

The performance of an online learning algorithm is typically assessed
using various notions of \emph{regret}.  
\emph{Static regret}, $\mathcal{R}_s$, measures the difference between the algorithm's cumulative loss 
and the cumulative loss of the best fixed decision in hindsight
\cite{cesa2006prediction}: 
\begin{equation*}
\mathcal{R}_s = \sum\limits_{t=1}^T f_t(\theta_t) 
- \min\limits_{\theta\in \cS}\sum\limits_{t=1}^T f_t(\theta),
\end{equation*}
where $\cS$ is a constraint set. For convex functions, variations of
gradient descent achieve static regret of $O(\sqrt{T})$, while for strongly
convex functions these can be improved to $O(\log T)$
\cite{hazan2016introduction}. 
However, when the underlying environment is changing, 
due to the fixed comparator the algorithms converge to, static regret is no longer appropriate.

In order to better track the changes of the underlying environment, 
\emph{dynamic regret} is proposed to compare the cumulative loss against that
incurred by a comparison sequence, $z_1,\ldots,z_T\in \cS$:
\begin{equation*}
\mathcal{R}_d = \sum\limits_{t=1}^T f_t(\theta_t) 
- \sum\limits_{t=1}^T f_t(z_t)
\end{equation*}
The classic work on online gradient descent \cite{zinkevich2003online} achieves dynamic regret
of the order $O(\sqrt{T}(1+V))$, where $V$ is a bound on the
path length of the comparison sequence:
\begin{equation*}
  \sum_{t=2}^{T} \|z_{t}-z_{t-1}\| \le V.
\end{equation*}
This has been improved to $O(\sqrt{T(1+V)})$ in \cite{zhang2018adaptive} by applying a
meta-optimization over step sizes. 

In works 
such as \cite{mokhtari2016online,yang2016tracking}, it is assumed that
$z_t = \theta_t^* = \argmin_{\theta \in \cS} f_t(\theta)$.
We denote that particular version of dynamic regret by:
\begin{equation*}
\mathcal{R}_d^* = \sum\limits_{t=1}^T f_t(\theta_t) 
- \sum\limits_{t=1}^T f_t(\theta^*_t)
\end{equation*}
In particular, if $V^*$ is the corresponding path length:
\begin{equation}
  \label{eq:optimizerLength}
 V^* = \sum\limits_{t=2}^T\left\|\theta_t^*-\theta_{t-1}^*\right\|,
\end{equation}
then \cite{mokhtari2016online} shows that for strongly convex functions, $\mathcal{R}_d^*$ of order
$O(V^*)$ is obtained by gradient descent.
However, as pointed out by \cite{zhang2018adaptive},
$V^*$ metric is too pessimistic and unsuitable for stationary problems,
which will result in poor generalization due to the
random perturbation caused by sampling from the \emph{same} distribution.
Thus, a trade-off between static regret $\cR_s$ and dynamic regret $\cR_d^*$
is desired to maintain the abilities of both generalization to stationary problem and tracking to the local changes.

\emph{Adaptive regret} \cite{hazan2009efficient} is another metric when dealing with changing environment,
which is defined as the maximum static regret over any contiguous time interval.
Although it shares the similar goal as the dynamic regret, their relationship is still an open question.

Closely related to the problem of online learning is adaptive
filtering, in which time series data is predicted using a filter that
is designed from past data \cite{sayed2011adaptive}. The performance
of adaptive filters is typically measured in an average case setting
under statistical assumptions.
One of the most famous adaptive filtering techniques is recursive least
squares, which bears strong resemblance to the online Newton method of
\cite{hazan2007logarithmic}. The work in \cite{hazan2007logarithmic}
proves a static regret bound of $O(\log T)$ for online Newton methods, but dynamic regret bounds
are not known. 

In order to have an algorithm that adapts to non-stationary data, 
it
is common to use a forgetting factor. For the recursive least squares,
\cite{guo1993performance} analyzed the effect of the forgetting factor
in terms of the tracking error covariance matrix,
and \cite{zhao2019distribution} made the tracking error analysis
with the assumptions that the noise is sub-Gaussian and 
the parameter follows a drifting model.
However, none of the analysis mentioned is done in terms of the regret,
which eliminates any noise assumption.
For the online learning,
\cite{garivier2011upper} analyzed the discounted UCB,
which uses the discounted empirical average as the estimate for the upper confidence bound.
\cite{russac2019weighted} used the weighted least-squares to update the linear bandit's underlying parameter.

The contributions of this paper are:
\begin{enumerate}
\item For exp-concave and strongly convex problems, we propose a discounted
  Online Newton algorithm  which generalizes recursive least squares
  with forgetting factors and the original online Newton method of
  \cite{hazan2007logarithmic}. We show how tuning the forgetting
  factor can achieve a dynamic regret bound of $\cR_d \le \max\{O(\log
  T),O(\sqrt{TV})\}$. This gives a rigorous analysis of forgetting
  factors in recursive least squares and improves the bounds described in \cite{zhang2018adaptive}.
  However, this choice requires a bound on the path length, $V$.
  For an alternative choice of forgetting factors, which does not
  require path length knowledge, we can
  simultaneously bound static regret by
  $\cR_s\le O(T^{1-\beta})$ and dynamic regret by $\cR_d\le
  \max\{O(T^{1-\beta}),O(T^\beta V)\}$. Note that tuning $\beta$
  produces a trade-off between static and dynamic regret. 
\item Based on the analysis of discounted recursive least squares, we
  derive a novel step size rule for online gradient descent. 
  Using this step size rule for smooth, strongly convex functions we
  obtain a static regret bound of
$\cR_s\le O(T^{1-\beta})$ and  a dynamic regret bound against
$\theta_t = \argmin_{\theta \in \cS} f_t(\theta)$ of $\cR_d^*\le
O(T^{\beta}(1+V^*))$. This improves the trade-off obtained in the
exp-concave case, since static regret or dynamic regret can be made
small by appropriate choice of $\beta \in (0,1)$. 
\item We show how the step size rule can be modified further so that
  gradient descent recovers the $\max\{O(\log T),O(\sqrt{TV})\}$
  dynamic regret bounds obtained by discounted Online Newton
  methods. However, as above, these bounds require knowledge of the
  bound on the path length, $V$.
\item Finally, we describe a meta-algorithm, similar to that used in
  \cite{zhang2018adaptive}, which can recover the $\max\{O(\log
  T),O(\sqrt{TV})\}$ dynamic regret bounds without knowledge of
  $V$. These bounds are tighter than those in
  \cite{zhang2018adaptive}, since they exploit exp-concavity to reduce
  the loss incurred by running an experts algorithm. 
  Furthermore, we give a lower bound for the corresponding problems,
  which matches the obtained upper bound for certain range of $V$.

\end{enumerate}

\paragraph{Notation.}

For the $n$ dimensional vector $\theta\in \mathbb{R}^n$, we use $\left\|\theta\right\|$ to denote the $\ell_2$-norm.
The gradient of the function $f_t$ at time step $t$ in terms of the $\theta$ is denoted as $\nabla f_t(\theta)$.

For the matrix $A\in \mathbb{R}^{m\times n}$, its transpose is denoted by $A^\top$ and $A^\top A$ denotes the matrix multiplication.
The inverse of $A$ is denoted as $A^{-1}$. 
When $m=n$, we use $\left\|A\right\|_2$ to represent the induced $2$ norm of the square matrix.
For the two square matrix $A\in\mathbb{R}^{n\times n}$ and $B\in\mathbb{R}^{n\times n}$,
$A\preceq B$ means $A-B$ is negative semi-definite, 
while $A\succeq B$ means $A-B$ is positive semi-definite.
For a positive definite matrix, $M$, let $\|x\|_M^2 = x^\top M x$. The
standard inner product between matrices is given by  $\langle A,B\rangle =
\Tr(A^\top B)$. The  determinant of a square matrix, $A$ is denoted by
$|A|$. 
We use $I$ to represent the identity matrix.

\section{Discounted Online Newton Algorithm}

As described above, the online Newton algorithm from
\cite{hazan2007logarithmic} strongly resembles the classic recursive
least squares algorithm from adaptive filtering
\cite{sayed2011adaptive}. Currently, only the static regret of
the online Newton method is studied. To obtain more adaptive
performance, forgetting factors are often used in recursive least
squares. However, the regret of forgetting factor algorithms has not
been analyzed. This section proposes a class of algorithms that
encompasses recursive least squares with forgetting factors and the
online Newton algorithm. We show how dynamic regret bounds for these
methods can be obtained by tuning the forgetting factor.


First we describe the problem assumptions. Throughout the paper we
assume that $f_t : \cS \to \bbR$ are convex, differentiable functions, $\cS$ is a
compact convex set, $\|x\|\le D$ for all $x\in \cS$, and $\|\nabla
f_t(x)\| \le G$ for all $x\in \cS$.  
Without loss of generality, we assume throughout the paper that $D\ge1$.

In this section we assume that all of the objective functions,
$f_t:\cS\to \bbR$
are $\alpha$-exp-concave for some $\alpha >0$. This means that $e^{-\alpha f_t(\theta)}$ is
concave.

If $f_t$ is twice differentiable, it can be shown that $f_t$ is 
$\alpha$-exp-concave if and only if
  \begin{equation}
    \label{expHess}
  \nabla^2 f_t(x) \succeq \alpha
  \nabla f_t(x) \nabla f_t(x)^\top
  \end{equation}
  for all $x\in \cS$. 

For an $\alpha$-exp-concave function $f_t$, Lemma 4.2 of \cite{hazan2016introduction} implies that for all
$\rho \le \frac{1}{2}\min\{\frac{1}{4GD},\alpha\}$, the following
bound holds for all $x$ and $y$ in $\cS$:
\begin{subequations}
  \label{eq:functionBounds}
  \begin{equation}
    \label{eq:expBound}
     f_t(y)\ge f_t(x)+\nabla f_t(x)^\top (y-x)+ 
\frac{\rho}{2}(x-y)^\top\nabla f_t(x)\nabla
    f_t(x)^\top(x-y) .
  \end{equation}

  In some variations on the algorithm, we will require extra
  conditions on the function, $f_t$. 
  In particular, in one variation
  we will require $\ell$-strong convexity.
  which means that there is a number
  $\ell >0$ such that
  \begin{align}
    \label{eq:strongConvex}
 f_t(y)\ge f_t(x)+\nabla f_t(x)^\top (y-x) + \frac{\ell}{2}\|x-y\|^2 
  \end{align}
  for all $x$ and $y$ in $\cS$. For twice-differentiable functions,  
  strong convexity implies $\alpha$-exp-concavity for $\alpha \le
  \ell /G^2$ on $\cS$.

  In another variant, we will require that the following bound holds
  for all $x$ and $y$ in $\cS$: 
  \begin{align}
    \label{eq:quadBound}
    f_t(y)\ge f_t(x) + \nabla f_t(x)^\top (y-x) + \frac{1}{2}
      \|x-y\|_{\nabla^2 f_t(x)}^2.
\end{align}
This bound does not correspond to a commonly used convexity class, but
it does hold for the important special case of quadratic functions:
$f_t(x) = \frac{1}{2}\|y_t - A_t x\|^2$. This fact will be important for
analyzing the classic discounted recursive least-squares
algorithm. Note that if $y_t$ and $A_t$ are restricted to compact
sets, $\alpha$ can be chosen so that $f_t$ is $\alpha$-exp-concave.

Additionally, the algorithms for strongly convex functions and those
satisfying (\ref{eq:quadBound}) will require that the gradients
$\nabla f_t(x)$ are $u$-Lipschitz for all $x\in \cS$ (equivalently, $f_t(x)$ is $u$-smooth),
which means the gradient $\nabla f_t(x)$ satisfies the relation 
\begin{equation*}
\left\|\nabla f_t(x)-\nabla f_t(y)\right\| \le u\left\|x-y\right\|, \forall t.
\end{equation*}
This smoothness condition is  equivalent to 
$f_t(y)\le f_t(x)+\nabla
f_t(x)^T(y-x)+\frac{u}{2}\left\|y-x\right\|^2$ and implies, in 
particular, that $\nabla^2 f_t(x) \preceq u I$.
\end{subequations}

\begin{algorithm}
\caption{Discounted Online Newton Step}
\label{alg:discountedNewton}
  \begin{algorithmic}
    \STATE{Given constants $\epsilon >0$, $\eta >0$, and $\gamma \in (0,1)$.}
    \STATE{Let $\theta_1 \in \cS$ and $P_0 = \epsilon I$.}
    \FOR{t=1,\ldots,T}
    \STATE{ Play $\theta_t$ and incur loss $f_t(\theta_t)$ }
    \STATE{ Observe $\nabla_t = \nabla f_t(\theta_t)$ and $H_t =
      \nabla^2 f_t(\theta_t)$ (if needed)}
    \STATE{ Update $P_t$:
      \begin{subequations}
        \begin{align}
          \label{eq:quasiP}
                            P_t &= \gamma P_{t-1} + \nabla_t \nabla_t^\top
                            && \textrm{(Quasi-Newton)}
                            \\
          \label{eq:fullP}
                            P_t &= \gamma P_{t-1} + H_t && \textrm{(Full-Newton)}
                          \end{align}
                          \end{subequations}
                        }
   \STATE{ Update $\theta_t$: $\theta_{t+1} = \Pi_{\cS}^{P_t}\left(\theta_t
       -\frac{1}{\eta} P_t^{-1}\nabla_t\right)$}
    \ENDFOR
  \end{algorithmic}
  
\end{algorithm}

To accommodate these three different cases, we propose Algorithm \ref{alg:discountedNewton},
in which $\Pi_\cS^{P_t}(y) = \argmin_{z\in \cS} \|z-y\|_{P_t}^2$
is the projection onto $\cS$ with respect to the norm induced by $P_t$. 


By using Algorithm \ref{alg:discountedNewton}, the following theorem can be obtained:
\begin{theorem}
  \label{thm:expConcaveThm}
  
  Consider the following three cases of Algorithm~\ref{alg:discountedNewton}:
  \begin{enumerate} 
  \item \label{it:exp} $f_t$ is $\alpha$-exp-concave. The algorithm
    uses $\eta \le
    \frac{1}{2}\min\{\frac{1}{4GD},\alpha\}$, $\epsilon = 1$
    \footnote{The value used here is only for proof simplicity, please see Meta-algorithm Section for more discussion.}, 
    and \eqref{eq:quasiP}. 
  \item  \label{it:strong}
    $f_t$ is $\alpha$-exp-concave and $\ell$-strongly convex while
    $\nabla f_t(x)$ is $u$-Lipschitz. The algorithm uses $\eta \le
    \ell / u$, $\epsilon = 1$, and \eqref{eq:fullP}.
  \item $f_t$ is $\alpha$-exp-concave and satisfy
    (\ref{eq:quadBound}) while $\nabla f_t(x)$ is
    $u$-Lipschitz. The algorithm uses $\eta \le 1$, $\epsilon = 1$, and \eqref{eq:fullP}.
  \label{it:quad}
\end{enumerate}
For each of these cases, there are positive constants $a_1,\ldots a_4$ such that  
  \begin{equation*}
  \begin{array}{ll}
    \sum_{t=1}^T (f_t(\theta_t)-f_t(z_t)) \le -a_1 T \log \gamma -a_2\log(1-\gamma)
     + \frac{a_3}{1-\gamma} V + a_4
   \end{array}
 \end{equation*}
 for all $z_1,\ldots,z_T\in \cS$ such that $\sum_{t=2}^T
 \|z_t-z_{t-1}\| \le V$.
\end{theorem}

Due to space limits, the proof is in the Appendix. Now we describe
some consequences of the theorem. 

\begin{corollary}
\label{cor::newton_static}
  Setting $\gamma = 1-T^{-\beta}$ with $\beta\in(0,1)$ leads to the following form:
  \begin{equation*}
  \begin{array}{ll}
   \sum_{t=1}^T (f_t(\theta_t)-f_t(z_t))
   \le O(T^{1-\beta} + \beta \log T+ T^{\beta} V)
   \end{array}
 \end{equation*}
\end{corollary}

\begin{proof}
  The first term is bounded as:
  \begin{align*}
    -T \log \gamma = -T \log(1-T^{-\beta}) 
                    \le \frac{T^{1-\beta}}{1-T^{-\beta}} = O(T^{1-\beta}),
  \end{align*}
  where the inequality follows from $-\log(1-x) \le
  \frac{x}{1-x}$ for $0 \le x < 1$. 

  The other terms follow by direct calculation.
\end{proof}

This corollary guarantees that the static regret is bounded in the
order of $O(T^{1-\beta})$ since $V=0$ in that case. The dynamic regret is of order
$O(T^{1-\beta}+T^{\beta} V)$. By choosing $\beta \in (0,1)$, we are
guaranteed that both the static and dynamic regrets are both sublinear
in $T$ as long as $V< O(T)$. Also, small static regret can be obtained by setting
$\beta$ near $1$.

In the setting of Corollary~\ref{cor::newton_static}, the algorithm
parameters do not depend on the path length $V$. Thus, the bounds hold
for any path length, whether or not it is known a priori.
The next corollary shows how tighter bounds could be obtained if
knowledge of $V$ were exploited in choosing the discount factor,
$\gamma$. 

\begin{corollary}
  \label{cor:logBounds}
Setting $\gamma = 1-\frac{1}{2}\sqrt{\frac{\max\{V,\log^2 T/T\}}{2DT}}$
leads to the form:
\begin{equation*}
\sum_{t=1}^T (f_t(\theta_t)-f_t(z_t)) \le \max\{O(\log T),O(\sqrt{TV})\}
\end{equation*}
\end{corollary}

The proof is similar to the proof of Corollary \ref{cor::newton_static}.

Note that Corollary~\ref{cor:logBounds} implies that the discounted
Newton method achieves logarithmic static regret by setting
$V=0$. This matches the bounds obtained in
\cite{hazan2007logarithmic}. For positive path lengths bounded by $V$,
we improve the $O(\sqrt{T(1+V)})$ dynamic bounds from
\cite{zhang2018adaptive}. However, the algorithm above current requires
knowing a bound on the path length,
whereas \cite{zhang2018adaptive} achieves its bound without knowing
the path length, a priori.  

If we view $V$ as the variation budget that $z_1^T = {z_1,\dots,z_T}$ can vary over $\cS$ like in \cite{besbes2015non},
and use this as a pre-fixed value to allow the comparator sequence to vary arbitrarily
over the set of admissible comparator sequence 
$\{z_1^T\in \cS:\sum\limits_{t=2}^T\left\|z_t-z_{t-1}\right\|\le V \}$,
we can tune $\gamma$ in terms of $V$.

In order to bound the dynamic regret
without knowing a bound on the path length, the method of
\cite{zhang2018adaptive} runs a collection of gradient descent algorithms
in parallel with different step sizes and then uses a meta-optimization
\cite{cesa2006prediction} to weight their solutions. In a later section,
we will show how a related meta-optimization over the discount factor
leads to 
$\max\{O(\log T),O(\sqrt{TV})\}$ dynamic regret bounds for unknown
$V$.

For the Algorithm \ref{alg:discountedNewton}, 
we need to invert $P_t$,
which can be achieved in time $O(n^2)$  for the Quasi-Newton case
in \eqref{eq:quasiP} by utilizing the matrix inversion lemma.
However, for the Full-Newton step \eqref{eq:fullP}, the inversion
requires  $O(n^3)$ time.

%

\section{From Forgetting Factors to a Step Size Rule}

In the next few sections, we aim to derive gradient descent rules
that achieve similar static and regret bounds to the discounted Newton
algorithm, without the cost of inverting matrices. We begin by
analyzing the special
case of quadratic functions of the form:
\begin{equation}
\label{eq::quadratic_loss}
f_t(\theta) = \frac{1}{2}\left\|\theta - y_t\right\|^2,
\end{equation}
where $y_t \in \cS$.
In this case, we will see that discounted recursive least squares
can be interpreted as online gradient descent with a special
step size rule.
We will show how this step size rule achieves a trade-off between
static regret and dynamic regret with the specific comparison sequence $\theta_t^*
=y_t= \argmin_{\theta \in \cS} f_t(\theta)$.
For a related analysis of more general
quadratic functions, $f_t(\theta)=
\frac{1}{2}\|A_t \theta - y_t\|^2$, please see the appendix.

Note that the previous section focused on dynamic regret for arbitrary
comparison sequences, $z_1^T \in \cS$. The analysis techniques in this
and the next section are specialized to comparisons against
$\theta_t^* = \argmin_{\theta\in\cS} f_t(\theta)$, as studied in works
such as \cite{mokhtari2016online,yang2016tracking}.  

Classic discounted recursive least squares corresponds to
Algorithm~\ref{alg:discountedNewton} running with full Newton steps, $\eta =
1$, and initial matrix $P_0=0$.  When $f_t$ is defined as in
\eqref{eq::quadratic_loss}, we have that $P_t = \sum_{k=0}^{t-1}
\gamma^k I$. Thus, the update rule can be expressed in the following
equivalent ways:
\begin{subequations}
\label{eq::quad_dis_update}
\begin{align}
\theta_{t+1} & = \argmin_{\theta\in\cS}\sum\limits_{i=1}^{t}\gamma^{i-1} f_{t+1-i}(\theta) \\
             & = \frac{\gamma-\gamma^t}{1-\gamma^t}\theta_t + \frac{1-\gamma}{1-\gamma^t}y_t\\
             &= \theta_t - P_t^{-1}\nabla f_t(\theta_t) \\
             & = \theta_t - \eta_t \nabla f_t(\theta_t),
\end{align}
\end{subequations}
where $\eta_t = \frac{1-\gamma}{1-\gamma^t}$. 
Note that since $y_t \in \cS$, no projection steps are needed.

The above update is the ubiquitous gradient descent with a changing step size.
The only difference from the standard methods is the choice of $\eta_t$,
which will lead to the useful trade-off between dynamic regret $\cR_d^*$ and static regret
to maintain the abilities of both generalization to stationary problem
and tracking to the local changes.

By using the above update, we can get the relationship between 
$\theta_{t+1}-\theta_t^*$ and $\theta_t - \theta_t^*$ as the following result:
\begin{lemma}
  \label{lem:quadratic_var_path_length}
  {\it
  Let $\theta_t^* = \argmin_{\theta \cS} f_t(\theta)$ in Eq.\eqref{eq::quadratic_loss}. 
  When using the discounted recursive least-squares update in Eq.\eqref{eq::quad_dis_update},
  we have the following relation:
  \begin{equation*}
  \theta_{t+1}-\theta_t^* = \frac{\gamma-\gamma^t}{1-\gamma^t}(\theta_{t} - \theta_t^*)
  \end{equation*}

  }
\end{lemma}

\begin{proof}
Since $\theta_t^*$ $=$ $\argmin f_t(\theta)$ $=y_t$, for $\theta_{t+1}-\theta_t^*$, we have:
\begin{equation*}
\begin{array}{ll}
\theta_{t+1}-\theta_t^* & = \theta_{t+1} - y_t\\
                        & = \frac{\gamma-\gamma^t}{1-\gamma^t}\theta_t + \frac{1-\gamma}{1-\gamma^t}y_t - y_t\\
                        & = \frac{\gamma-\gamma^t}{1-\gamma^t}(\theta_t - y_t)
                         = \frac{\gamma-\gamma^t}{1-\gamma^t}(\theta_t - \theta_t^*)
\end{array}
\end{equation*}

\end{proof}

Recall from \eqref{eq:optimizerLength} that the path length of optimizer sequence is denoted by
$V^*$. 
With the help of Lemma \ref{lem:quadratic_var_path_length}, 
we can upper bound the dynamic regret $\cR_d^*$
in the next theorem:

\begin{theorem}
\label{thm::quad_dynamic_regret}
{\it Let $\theta_t^*$ be the solution to $f_t(\theta)$ in Eq.\eqref{eq::quadratic_loss}. 
  When using the discounted recursive least-squares update in Eq.\eqref{eq::quad_dis_update} 
  with $1-\gamma = 1/T^{\beta}, \beta\in (0,1)$,
  we can upper bound the dynamic regret as:
  \begin{equation*}
  \mathcal{R}_d^* \le 2DT^{\beta}\big(\left\|\theta_1-\theta_1^*\right\|+V^*\big)
  \end{equation*}

}
\end{theorem}

\begin{proof}
According to the Mean Value Theorem, 
there exists a vector $x\in \{v| v = \delta \theta_t + (1-\delta)\theta_t^*,\delta\in[0,1]\}$
such that $f_t(\theta_t)-f_t(\theta_t^*) = \nabla f_t(x)^T(\theta_t-\theta_t^*)\le \left\|\nabla f_t(x)\right\| \left\|\theta_t-\theta_t^*\right\|$. 
For our problem, $\left\|\nabla f_t(x)\right\| = \left\|x-y_t\right\|\le\left\|x\right\|+\left\|y_t\right\|$.
For $\left\|x\right\|$, we have:
\begin{equation*}
\begin{array}{ll}
\left\|x\right\| &= \left\|\delta \theta_t + (1-\delta)\theta_t^*\right\| \\
                 &\le \delta\left\|\theta_t\right\|+(1-\delta)\left\|y_t\right\| \\
                 & = \delta\left\|\frac{\sum\limits_{i=1}^{t-1}\gamma^{i-1}y_{t-i}}{\sum\limits_{i=1}^{t-1}\gamma^{i-1}}\right\|
                  + (1-\delta)\left\|y_t\right\| \\
                &\le D
\end{array}
\end{equation*}
where the second inequality is due to $\left\|y_i\right\|\le D, \forall i$.

As a result, the norm of the gradient can be upper bounded as $\left\|\nabla f_t(x)\right\| \le 2D$.
Then we have $\mathcal{R}_d^* = \sum\limits_{t=1}^T \Big(f_t(\theta_t)-f_t(\theta_t^*)\Big)\le 
2D\sum\limits_{t=1}^T\left\|\theta_t-\theta_t^*\right\|$. 
Now we could instead upper bound $\sum\limits_{t=1}^T\left\|\theta_t-\theta_t^*\right\|$,
which can be achieved as follows:
\begin{equation*}
\begin{array}{ll}
\sum\limits_{t=1}^T\left\|\theta_t-\theta_t^*\right\| 
&= \left\|\theta_1-\theta_1^*\right\| 
                                                           + \sum\limits_{t=2}^T\left\|\theta_t-\theta_{t-1}^*+\theta_{t-1}^*-\theta_t^*\right\| \\
&\le \left\|\theta_1-\theta_1^*\right\| + \sum\limits_{t=1}^{T-1}\left\|\theta_{t+1}-\theta_{t}^*\right\| + \sum\limits_{t=2}^T\left\|\theta_t^*-\theta_{t-1}^*\right\| \\
&= \left\|\theta_1-\theta_1^*\right\| + \sum\limits_{t=1}^{T-1}\frac{\gamma-\gamma^t}{1-\gamma^t}\left\|\theta_{t}-\theta_{t}^*\right\| + \sum\limits_{t=2}^T\left\|\theta_t^*-\theta_{t-1}^*\right\|\\
&\le \left\|\theta_1-\theta_1^*\right\| + \sum\limits_{t=1}^{T}\frac{\gamma-\gamma^t}{1-\gamma^t}\left\|\theta_{t}-\theta_{t}^*\right\| + \sum\limits_{t=2}^T\left\|\theta_t^*-\theta_{t-1}^*\right\|
\end{array}
\end{equation*}
where in the second equality, we substitute the result from Lemma \ref{lem:quadratic_var_path_length}.

From the above inequality, we get 
\begin{equation*}
\sum\limits_{t=1}^T\Big(1-\frac{\gamma-\gamma^t}{1-\gamma^t}\Big)\left\|\theta_t-\theta_t^*\right\|
\le \left\|\theta_1-\theta_1^*\right\| + \sum\limits_{t=2}^T\left\|\theta_t^*-\theta_{t-1}^*\right\|
\end{equation*}

Since $\Big(1-\frac{\gamma-\gamma^t}{1-\gamma^t}\Big) = \frac{1-\gamma}{1-\gamma^t}\ge 1-\gamma$, we get
\begin{equation*}
\begin{array}{ll}
\sum\limits_{t=1}^T\left\|\theta_t-\theta_t^*\right\|
&\le \frac{1}{1-\gamma}\left\|\theta_1-\theta_1^*\right\| + \frac{1}{1-\gamma}\sum\limits_{t=2}^T\left\|\theta_t^*-\theta_{t-1}^*\right\|\\
&= T^{\beta}(\left\|\theta_1-\theta_1^*\right\|+\sum\limits_{t=2}^T\left\|\theta_t^*-\theta_{t-1}^*\right\|)
\end{array}
\end{equation*}
Thus, $\mathcal{R}_d \le 2D\sum\limits_{t=1}^T\left\|\theta_t-\theta_t^*\right\|
\le 2DT^{\beta}(\left\|\theta_1-\theta_1^*\right\|+\sum\limits_{t=2}^T\left\|\theta_t^*-\theta_{t-1}^*\right\|)$.
\end{proof}

Theorem \ref{thm::quad_dynamic_regret} shows that if we choose the
discounted factor $\gamma = 1- T^{-\beta}$ we obtain a dynamic regret
of $O(T^{\beta}(1+V^*))$. This is a refinement of the Corollary
\ref{cor::newton_static} since the bound no longer
has the $T^{1-\beta}$ term.  Thus, the dynamic regret can be made
small by choosing a small $\beta$. 


In the next theorem, we will show that this carefully chosen $\gamma$ can also lead to useful static regret,
which can give us a trade-off and solve
the dilemma of generalization for stationary problems 
versus the tracking for local changes.  \begin{theorem}
\label{thm::quad_static_regret}
{\it Let $\theta^*$ be the solution to $\min\sum\limits_{t=1}^T f_t(\theta)$. 
  When using the discounted recursive least-squares update in Eq.\eqref{eq::quad_dis_update} 
  with $1-\gamma = 1/T^{\beta}, \beta\in (0,1)$,
  we can upper bound the static regret as:
  \begin{equation*}
  \mathcal{R}_s \le O(T^{1-\beta})
  \end{equation*}
}
\end{theorem}

Recall that the algorithm of this section can be interpreted both as a
discounted recursive least squares method, and as a gradient descent
method. As a result, this theorem is actually a direct consequence of
Corollary~\ref{cor::newton_static}, by setting $V=0$. However, we will
give a separate proof in the appendix, since the techniques extend naturally to the
analysis of more general work on gradient descent methods of the next
section.

Our Theorems \ref{thm::quad_dynamic_regret} and \ref{thm::quad_static_regret} 
build a trade-off between dynamic and static regret
by the carefully chosen discounted factor $\gamma$.
Compared with the result from the last section, there are two improvements:
1. The two regrets are decoupled so that we could reduce the $\beta$
to make the dynamic regret $\cR_d^*$ result smaller than bound from Corollary~\ref{cor::newton_static};
2. The update is the first-order gradient descent, which is
computationally more efficient than second order methods.

In the next section, we will consider the strongly convex and smooth case,
whose result is inspired by this section's analysis.

\section{Online Gradient Descent for Smooth, Strongly Convex Problems}

In this section, we generalize the results of the previous section
idea to functions which are
$\ell$-strongly convex and $u$-smooth. We will see that similar
bounds on $\cR_s$ and $\cR_d^*$ can be obtained.


Our proposed update rule for the prediction $\theta_{t+1}$ at time step $t+1$ is:
\begin{equation}
\label{eq::gen_prob_update}
\theta_{t+1} = \argmin\limits_{\theta\in\cS}\left\|\theta - (\theta_t-\eta_t \nabla f_t(\theta_t))\right\|^2
\end{equation}
where $\eta_t = \frac{1-\gamma}{\ell(\gamma-\gamma^t)+u(1-\gamma)}$ and $\gamma \in (0,1)$.

This update rule generalizes the step size rule  from the last section.

Before getting to the dynamic regret, we will first derive the relation between
$\left\|\theta_{t+1}-\theta_{t}^*\right\|$ and $\left\|\theta_t-\theta_t^*\right\|$
to try to mimic the result in Lemma \ref{lem:quadratic_var_path_length} of the quadratic case:
\begin{lemma}
  \label{lem:gen_prob_var_path}
  {\it
  Let $\theta_t^*\in\cS$ be the solution to $f_t(\theta)$ which is strongly convex and smooth.
  When we use the update in Eq.\eqref{eq::gen_prob_update},
  the following relation is obtained:
  \begin{equation*}
  \left\|\theta_{t+1} -\theta_t^*\right\| \le \sqrt{1-\frac{l(1-\gamma)}{u(1-\gamma)+l\gamma}} \left\|\theta_t-\theta_t^*\right\|
  \end{equation*}
  }
\end{lemma}
Due to space limits, please refer to the appendix for the proof.

Now we are ready to present the dynamic regret result:
\begin{theorem}
\label{thm::gen_prob_dynamic_regret}
{\it 
Let $\theta_t^*$ be the solution to $f_t(\theta),\theta\in\cS$.
When using the update in Eq.\eqref{eq::gen_prob_update}
with $1-\gamma = 1/T^{\beta}, \beta \in (0,1)$,
we can upper bound the dynamic regret:
  \begin{equation*}
  \mathcal{R}_d^* \le G\big(2(T^{\beta}-1)+u/l\big)(\left\|\theta_1-\theta_1^*\right\|
+ V^*)
  \end{equation*}
}
\end{theorem}

The proof follows the similar steps in the proof of Theorem \ref{thm::quad_dynamic_regret}.
Due to space limits, please refer to the appendix.

Theorem \ref{thm::gen_prob_dynamic_regret}'s result seems promising in achieving the trade-off,
since it has a similar form of the result from quadratic problems in Theorem \ref{thm::quad_dynamic_regret}.
Next, we will present the static regret result, which assures that the
desired trade-off can be obtained.

\begin{theorem}
\label{thm::gen_prob_static_regret}
{\it Let $\theta^*$ be the solution to $\min\limits_{\theta\in\cS}\sum\limits_{t=1}^T f_t(\theta)$. 
  When using the update in Eq.\eqref{eq::gen_prob_update} with $1-\gamma = 1/T^{\beta}, \beta\in (0,1)$,
  we can upper bound the static regret:
  \begin{equation*}
  \mathcal{R}_s \le O(T^{1-\beta})
  \end{equation*}
}
\end{theorem}

The proof follows the similar steps in the proof of Theorem \ref{thm::quad_static_regret}.
Due to space limits, please refer to the appendix.

The regret bounds of this section are similar to those obtained for
simple quadratics. Thus, this gradient descent rule maintains all of
the advantages over the discounted Newton method that were described
in the previous section
and
the advantages of trading off static regret and dynamic regret $\cR_d^*$.

\section{Online Gradient Descent for Strongly Convex Problems}

In this section, we extend step size idea from previous section
to problems which are $\ell$-strongly convex, but not necessarily
smooth. We obtain a dynamic regret of $\cR_d\le \max\{O(\log
T),O(\sqrt{TV})\}$, similar to the discounted online Newton
method. However, our analysis does not lead to the clean trade-off of
$\cR_s \le O(T^{1-\beta})$ and $\cR_d^* \le O(T^{\beta}(1+V^*))$
obtained when smoothness is also used.

The update rule is online gradient descent:
\begin{equation}
\label{eq::strongly_update}
\theta_{t+1} = \argmin\limits_{\theta\in\cS}\left\|\theta - (\theta_t-\eta_t \nabla f_t(\theta_t))\right\|^2
\end{equation}
where $\eta_t = \frac{1-\gamma}{\ell(1-\gamma^t)}$,
and $\gamma \in (0,1)$.

We can see that the update rule is the same as the one in Eq.\eqref{eq::gen_prob_update}
while the step size $\eta_t$ is replaced with $\frac{1-\gamma}{\ell(1-\gamma^t)}$.


By using the new step size with the update rule in Eq.\eqref{eq::strongly_update},
we can obtain the following dynamic regret bound:
\begin{theorem}
\label{thm::strongly_regret}
{\it
If using the update rule in Eq.\eqref{eq::strongly_update} 
with $\eta_t = \frac{1-\gamma}{\ell(1-\gamma^t)}$ and $\gamma \in (0,1)$, 
the following dynamic regret can be obtained:
\begin{equation*}
\sum\limits_{t=1}^T \Big(f_t(\theta_t)-f_t(z_t)\Big) 
\le 2D\ell \frac{1}{1-\gamma}V + \frac{G^2}{2}\sum\limits_{t=1}^T\eta_t
\end{equation*}
  
}
\end{theorem}
\begin{proof}
According to the non-expansive property of the projection operator and the update rule in Eq.\eqref{eq::strongly_update},
we have
\begin{equation*}
\begin{array}{ll}
\left\|\theta_{t+1}-z_t\right\|^2& \le \left\|\theta_t-\eta_t\nabla f_t(\theta_t)-z_t\right\|^2 \\
& = \left\|\theta_t-z_t\right\|^2 -2\eta_t\nabla f_t(\theta_t)^T(\theta_t-z_t)
+\eta_t^2\left\|\nabla f_t(\theta_t)\right\|^2
\end{array}
\end{equation*}
The reformulation gives us
\begin{equation} 
\label{eq::strongly_prob_contraction_ineq}
\begin{array}{ll}
\nabla f_t(\theta_t)^T(\theta_t-z_t)
\le \frac{1}{2\eta_t}\big(\left\|\theta_t-z_t\right\|^2 - \left\|\theta_{t+1}-z_t\right\|^2\big)
+ \frac{\eta_t}{2}\left\|\nabla f_t(\theta_t)\right\|^2
\end{array}
\end{equation}

Moreover, from the strong convexity, we have 
$f_t(z_t)\ge f_t(\theta_t)+\nabla f_t(\theta_t)^T(z_t-\theta_t)+\frac{\ell}{2}\left\|z_t-\theta_t\right\|^2$,
which is equivalent to 
$\nabla f_t(\theta_t)^T(\theta_t-z_t)\ge f_t(\theta_t)-f_t(z_t)+\frac{\ell}{2}\left\|z_t-\theta_t\right\|^2$.
Combined with Eq.\eqref{eq::strongly_prob_contraction_ineq}, we have
\begin{equation}
\label{eq::strongly_ineq_with_l}
\begin{array}{ll}
f_t(\theta_t)-f_t(z_t) \le \frac{1}{2\eta_t}\big(\left\|\theta_t-z_t\right\|^2 
- \left\|\theta_{t+1}-z_t\right\|^2\big)
+ \frac{\eta_t}{2}\left\|\nabla f_t(\theta_t)\right\|^2-\frac{\ell}{2}\left\|z_t-\theta_t\right\|^2
\end{array}
\end{equation}

Then we can lower bound $\|\theta_{t+1}-z_t\|^2$ by
\begin{equation}
\label{eq::strongly_lower_with_var}
\begin{array}{ll}
    \|\theta_{t+1}-z_t\|^2 & = \|\theta_{t+1}-z_{t+1}\|^2 +
                              \|z_{t+1}-z_t\|^2
                              +2 (\theta_{t+1}-z_{t+1})^\top
                              (z_{t+1}-z_t)\\
    
    & \ge \|\theta_{t+1}-z_{t+1}\|^2 - 4 D\|z_{t+1}-z_t\|
\end{array}
\end{equation}

Combining (\ref{eq::strongly_ineq_with_l}) and \eqref{eq::strongly_lower_with_var} gives
\begin{equation*}
\begin{array}{l}
f_t(\theta_t)-f_t(z_t) 
\le \frac{1}{2\eta_t}\big(\left\|\theta_t-z_t\right\|^2 
- \left\|\theta_{t+1}-z_{t+1}\right\|^2\big)+\frac{2D}{\eta_t}\|z_{t+1}-z_t\|
+ \frac{\eta_t}{2}\left\|\nabla f_t(\theta_t)\right\|^2-\frac{\ell}{2}\left\|z_t-\theta_t\right\|^2
\end{array}
\end{equation*}

Summing over $t$ from $1$ to $T$, dropping the term $-\frac{1}{2\eta_T}\|\theta_{T+1}-z_{T+1}\|^2$,
setting $z_{T+1}= z_T$, using the inequality $\|\nabla f_t(\theta_t)\|^2\le G^2$, and re-arranging gives
\begin{equation*}
\begin{array}{l}
\sum\limits_{t=1}^T \Big(f_t(\theta_t)-f_t(z_t)\Big) \\
\le \frac{1}{2}(\frac{1}{\eta_1}-\ell)\|\theta_1-z_1\|^2
+\frac{1}{2}\sum\limits_{t=1}^T(\frac{1}{\eta_t}-\frac{1}{\eta_{t-1}}-\ell)\|\theta_t-z_t\|^2 
 +2D\sum\limits_{t=1}^{T-1}\frac{1}{\eta_t}\|z_{t+1}-z_t\| + \frac{G^2}{2}\sum\limits_{t=1}^T\eta_t \\
 \le 2D\ell \frac{1}{1-\gamma}V + \frac{G^2}{2}\sum\limits_{t=1}^T\eta_t
\end{array}
\end{equation*}
where for the second inequality, we use the following results:
$\frac{1}{\eta_1}-\ell = 0$,
$\frac{1}{\eta_t}-\frac{1}{\eta_{t-1}}-\ell = \frac{\ell(1-\gamma)(\gamma^{t-1}-1)}{1-\gamma}\le 0$,
$\frac{1}{\eta_t} = \frac{\ell(1-\gamma^t)}{1-\gamma}\le \frac{\ell}{1-\gamma}$,
and the definition of $V$.

\end{proof}

Similar to the case of discounted online Newton methods, if a bound on
the path length, $V$, is known, the discount factor can be tuned to
achieve low dynamic regret: 
\begin{corollary}
\label{corr:strongly_dynamic_regret}
By setting $\gamma = 1-\frac{1}{2}\sqrt{\frac{\max\{V,\log^2 T/T\}}{2DT}}$,
the following bound can be obtained:
\begin{equation*}
\sum\limits_{t=1}^T \Big(f_t(\theta_t)-f_t(z_t)\Big)\le \max\{O(\log T),O(\sqrt{TV})\}.
\end{equation*}

\end{corollary} 

This result is tighter than the $O(\sqrt{T(1+V)})$ bound obtained
by \cite{zhang2018adaptive} on convex functions, but not directly
comparable to the $O(V^*)$ bounds obtained in
\cite{mokhtari2016online} for smooth, strongly convex functions.

Similar to the Corollary~\ref{cor:logBounds} on discounted online Newton methods, 
Corollary~\ref{corr:strongly_dynamic_regret} requires knowing $V$. In the next section, we will see how a
meta-algorithm can be used to obtain the same bounds without knowing $V$.

\section{Meta-algorithm}

In previous sections, we discussed the results on dynamic regret for
both $\alpha$-exp-concave and $\ell$-strongly convex objectives.
The tightest regret bounds were obtained by choosing a discount factor
that depends on $V$, a bound on the path length. 
In this section, we solve this issue 
by running multiple algorithms in parallel with different discount factors.

For online convex optimization, a similar meta-algorithm has been used
by \cite{zhang2018adaptive} to search over step sizes. However, the
method of \cite{zhang2018adaptive} cannot be used directly in either the $\alpha$-exp-concave or $\ell$-strongly convex case
due to the added $O(\sqrt{T})$ regret from running multiple
algorithms. In order to remove this factor, we exploit the
exp-concavity in the experts algorithm, as in Chapter 3 in \cite{cesa2006prediction}. 

In this section, we will show that by using appropriate parameters and analysis designed specifically for our cases,
the meta-algorithm can be used to solve our issues.

\begin{algorithm}
\caption{Meta-Algorithm}
\label{alg:meta}
  \begin{algorithmic}
    \STATE{Given step size $\lambda$, and a set $\cH$ containing
      discount factors for each algorithm.}
    \STATE{Activate a set of algorithms $\{A^\gamma|\gamma\in\cH\} $
    by calling Algorithm \ref{alg:discountedNewton} (exp-concave case) or the update in Eq.\eqref{eq::strongly_update} (strongly convex case)
    for each parameter $\gamma\in\cH$. }
    \STATE{Sort $\gamma$ in descending order $\gamma_1\ge\gamma_2\ge\dots\ge\gamma_N$, 
    and set $w_1^{\gamma_i} = \frac{C}{i(i+1)}$ with $C = 1+1/|\cH|$.}
    \FOR{t=1,\ldots,T}
    \STATE{Obtain $\theta_t^{\gamma}$ from each algorithm $A^{\gamma}$. }
    \STATE{Play $\theta_t = \sum\limits_{\gamma\in\cH}w_t^\gamma \theta_t^\gamma$,
    and incur loss $f_t(\theta_t^\gamma)$ for each $\theta_t^\gamma$. }
    \STATE{Update $w_t^\gamma$ by
            \begin{equation*}
               w_{t+1}^\gamma = \frac{w_t^\gamma\exp(-\lambda f_t(\theta_t^\gamma))}
               {\sum\limits_{\mu\in\cH}w_t^\mu \exp(-\lambda f_t(\theta_t^\mu))}.
            \end{equation*}
    \STATE{Send back the gradient $\nabla f_t(\theta_t^\gamma)$ for each algorithm $A^\gamma$.}
    }
    \ENDFOR
  \end{algorithmic}
\end{algorithm}

\subsection{Exp-concave case}

Before showing the regret result, we first show that the cumulative loss of the meta-algorithm
is comparable to all $A^\gamma\in\cH$:
\begin{lemma}\label{lem:meta-expert-compare}
If $f_t$ is $\alpha$-exp-concave and $\lambda = \alpha$, 
the cumulative loss difference of Algorithm \ref{alg:meta} for any $\gamma\in\cH$ is bounded as:
\begin{equation*}
\sum\limits_{t=1}^T (f_t(\theta_t) - f_t(\theta_t^\gamma))\le \frac{1}{\alpha}\log\frac{1}{w_1^\gamma}
\end{equation*}

\end{lemma}

This result shows how $O(\sqrt{T})$ regret incurred by running an
experts algorithm is reduced in the $\alpha$-exp-concave case. %
The result is similar to Proposition 3.1 of
\cite{cesa2006prediction}.

Based on the above lemma, if we can show that there exists an algorithm $A^\gamma$, 
which can bound the regret $\sum_{t=1}^T (f_t(\theta_t^\gamma)-f_t(z_t)) \le \max\{O(\log T), O(\sqrt{TV})\}$,
then we can combine these two results and show that the regret holds for $\theta_t,t=1,\dots,T$ as well:
\begin{theorem}
\label{thm:mega_exp-concave}
For any comparator sequence $z_1,\dots,z_T\in\cS$, 
setting $\cH=\Big\{\gamma_i = 1-\eta_i\Big|i=1,\dots,N\Big\}$ with $T\ge 2$
where $\eta_i = \frac{1}{2}\frac{\log T}{T\sqrt{2D}}2^{i-1}$, 
$N = \lceil \frac{1}{2}\log_2 (\frac{2DT^2}{\log^2 T})\rceil+1$,
and $\lambda = \alpha$ leads to the result:
\begin{equation*}
\sum_{t=1}^T (f_t(\theta_t)-f_t(z_t))\le O(\max\{\log T, \sqrt{TV}\})
\end{equation*}

\end{theorem}

As described previously, the proof's main idea is 
to show that we could both find an algorithm $A^\gamma$
bounding the regret $\sum_{t=1}^T (f_t(\theta_t^\gamma)-f_t(z_t)) \le \max\{O(\log T), O(\sqrt{TV})\}$
and
cover the $V$ with $O(\log T)$ different $\gamma$ choices.
Please see the appendix for the formal proof.

In practice, we include the additional case when $\gamma = 1$ to make the overall algorithm 
explicitly balance the static regret.
Also, the free parameter $\epsilon$ used in Algorithm \ref{alg:discountedNewton} is important
for the actual performance. If it is too small,
the update will be easily effected by the gradient to have high generalization error.
In practice, it can be set to be equal to $1/(\rho^2 D^2)$ or $1/(\rho^2 D^2 N)$ 
with $\rho = \frac{1}{2}\min\{\frac{1}{4GD},\alpha\}$ like in \cite{hazan2016introduction}.

\subsection{Strongly convex case}

For the strongly convex problem, since the parameter $\gamma$ 
used in Corollary \ref{corr:strongly_dynamic_regret} is the same as the one in Corollary \ref{cor:logBounds},
it seems likely that the 
meta-algorithm should work with the same setup in as Theorem
\ref{thm:mega_exp-concave}. The only parameter that needs to be
changed is $\lambda$, which was set above to $\alpha$, the parameter
of $\alpha$-exp-concavity. 

To proceed, we first show that the $\ell$-strongly convex function 
with bounded gradient (e.g.,$\left\|\nabla f_t\right\|\le G$)
is also $\ell/G^2$-exp-concave.
Previous works also pointed out this, 
but their statement only works when $f_t$ is second-order differentiable,
while our result is true when $f_t$ is first-order differentiable.

\begin{lemma}
\label{lem::strongly_is_exp}
For the $\ell$-strongly convex function $f_t$ with $\|\nabla f_t\|\le G$, 
it is also $\alpha$-exp-concave with $\alpha = \ell/G^2$.
\end{lemma}

Lemma \ref{lem::strongly_is_exp} indicates that running Algorithm \ref{alg:meta} with strongly convex function
leads to the same result as in Lemma \ref{lem:meta-expert-compare}.
Thus, using the similar idea as discussed in the case of $\alpha$-exp-concavity and Algorithm \ref{alg:meta},
the theorem below can be obtained:
\begin{theorem}
\label{thm:mega_strongly}
For any comparator sequence $z_1,\dots,z_T\in\cS$, 
setting $\cH=\Big\{\gamma_i = 1-\eta_i\Big|i=1,\dots,N\Big\}$ with $T\ge 2$
where $\eta_i = \frac{1}{2}\frac{\log T}{T\sqrt{2D}}2^{i-1}$, 
$N = \lceil \frac{1}{2}\log_2 (\frac{2DT^2}{\log^2 T})\rceil+1$,
and $\lambda = \ell/G^2$ leads to the result:
\begin{equation*}
\sum_{t=1}^T (f_t(\theta_t)-f_t(z_t))\le O(\max\{\log T, \sqrt{TV}\})
\end{equation*}
\end{theorem}

As discussed in the previous subsection,
in practice, we also include the case when $\gamma = 1$ to make the overall algorithm 
explicitly balance the static regret and set $\epsilon$ accordingly as in the exp-concave case.

\subsection{A Lower bound}
In the previous subsections, we demonstrate how to achieve the dynamic regret $\max\{O(\log T),O(\sqrt{TV})\}$
for both the exp-concave and strongly convex problems without knowing $V$.
In this subsection, we will give a lower bound,
which approaches the upper bound for large and small $V$.

\begin{proposition}
  \label{prop:lower-bound}
  For losses of the form $f_t(\theta) = (\theta - \epsilon_t)^2$, for
  all $\gamma_0 \in (0,1)$ and all $V=T^{\frac{2+\gamma_0}{4-\gamma_0}}$,
  there is a comparison sequence $z_1^T$
such that $\sum\limits_{t=2}^T\|z_t-z_{t-1}\|\le V$ and 
\begin{equation*}
\cR_d \ge \max\{O(\log T),O\big((VT)^{\frac{\gamma_0}{2}}\big)\}.
\end{equation*}

\end{proposition}

The above result has the following indications:
1. For $V = o(T)$ but approaching to $T$, it is impossible to achieve better bound of $\cR_d \ge O\Big((VT)^{\frac{\alpha_0}{2}}\Big)$
with $\alpha_0<1$.
2. For other ranges of $V$ like $V = O(\sqrt{T})$, its lower bound is not established and still an open question.

\section{Conclusion}

In this paper, we propose a discounted online Newton algorithm that
generalizes recursive least squares with forgetting factors and
existing online Newton methods. We prove a dynamic regret bound
$\max\{O(\log T),O(\sqrt{TV})\}$ which provides a rigorous
analysis of forgetting factor algorithms. 
In the special case of simple quadratic functions, we demonstrate that the
discounted Newton method reduces to a gradient descent algorithm with
a particular step size rule.
We show how this step size rule can be generalized to apply to
strongly convex functions, giving a substantially more
computationally efficient algorithm than the discounted online Newton
method, while recovering the dynamic regret guarantees.
The strongest regret guarantees depend on knowledge of the path
length, $V$. We show how to use a meta-algorithm that optimizes over
discount factors to obtain the same regret guarantees without
knowledge of $V$ as well as a lower bound which matches the obtained upper bound
for certain range of $V$.
Finally, when the functions are smooth we show how this
new gradient descent method enables a static regret of $\cR_s\le
O(T^{1-\beta})$ and $\cR_d^* \le O(T^{\beta} (1+V^*))$, where $\beta \in
(0,1)$ is a user-specified trade-off parameter.
%


\bibliographystyle{unsrt}
\bibliography{OCO_dynamic}

\newpage

\twocolumn[
\textbf{\large \centerline{Appendix:} 
}
]
\appendix

The supplementary material contains proofs of the some results of the
paper along with supporting results.

\paragraph{Proof of Theorem~\ref{thm:expConcaveThm}:}

Before proving the theorem, the following observation is
helpful. 

\begin{lemma}\label{lem:pBound}
  If $P_t$ is updated via \eqref{eq:quasiP} then $\|P_t\| \le \epsilon
  + \frac{G^2}{1-\gamma}$, while if $P_t$ is updated via
  $\eqref{eq:fullP}$, then $\|P_t\| \le \epsilon + \frac{u}{1-\gamma}$.
\end{lemma}
\begin{proof}
  First consider the quasi-Newton case. The bound holds at
  $P_0=\epsilon I$, so assume that it holds at time $t-1$ for $t\ge
  1$.  Then, by induction we have
  \begin{align*}
    \|P_t\| &= \|\gamma P_{t-1} + \nabla_t \nabla_t \| \\
            &\le \gamma \|P_{t-1}\| + G^2 \\
            &\le \gamma \epsilon + \frac{G^2}{1-\gamma} \\
    & \le \epsilon + \frac{G^2}{1-\gamma}.
  \end{align*}
  The full-Newton case is identical, except it uses the bound
  $\|H_t\|\le u$.
  
\end{proof}

  The generalized Pythagorean theorem implies that
  \begin{align*}
    \|\theta_{t+1}-z_t\|_{P_t}^2 & \le \left\|\theta_t -
                              \frac{1}{\eta}P_{t}^{-1} \nabla_t - z_t
                              \right\|_{P_t}^2\\
    &= \|\theta_t-z_t\|_{P_t}^2 + \frac{1}{\eta^2} \nabla_t^\top P_t^{-1}
      \nabla_t \\
    &\quad-\frac{2}{\eta} \nabla_t^\top (\theta_t-z_t).
  \end{align*}
  Re-arranging shows that
  \begin{align}
    \label{eq:Pythagorean}
    \nonumber
    \nabla_t^\top (\theta_t-z_t) &\le \frac{1}{2\eta}\nabla_t^\top P_t^{-1}
    \nabla_t + \frac{\eta}{2}\Big(\|\theta_t- z_t\|_{P_t}^2\\
    &\quad\quad-\|\theta_{t+1}-z_t\|_{P_t}^2\Big)
  \end{align}

  Let $c_1$ be the upper bound on $\|P_t\|$ from
  Lemma~\ref{lem:pBound}. Then we can lower bound
  $\|\theta_{t+1}-z_t\|_{P_t}^2$ by
  \begin{align}
    \nonumber
    \|\theta_{t+1}-z_t\|_{P_t}^2 & = \|\theta_{t+1}-z_{t+1}\|_{P_t}^2 +
                                   \|z_{t+1}-z_t\|_{P_t}^2\\
    \nonumber
                              &\quad+2 (\theta_{t+1}-z_{t+1})^\top P_t
                              (z_{t+1}-z_t)\\
    \label{eq:csBound}
    & \ge \|\theta_{t+1}-z_{t+1}\|_{P_t}^2 - 4 D c_1 \|z_{t+1}-z_t\|
  \end{align}

  Combining (\ref{eq:Pythagorean}) and \eqref{eq:csBound} gives
  \begin{multline}
    \nonumber
    \nabla_t^\top (\theta_t-z_t) \le \frac{1}{2\eta}\nabla_t^\top P_t^{-1}
    \nabla_t + 2 D c_1 \eta \|z_{t+1}-z_t\| \\
    \frac{\eta}{2}\left(\|\theta_t- z_t\|_{P_t}^2-\|\theta_{t+1}-z_{t+1}\|_{P_t}^2 \right)
  \end{multline}
  Summing over $t$, dropping the term $-\|\theta_{T+1}-z_{T+1}\|_{P_T}^2$, setting $z_{T+1} = z_T$,  and re-arranging gives
  \begin{multline}
    \label{eq:sum}
    \sum_{t=1}^T \nabla_t^\top (\theta_t-z_t) \le \sum_{t=1}^T \frac{1}{2\eta}\nabla_t^\top P_t^{-1}
    \nabla_t  + 2 D c_1\eta V \\
    +\frac{\eta}{2} \epsilon \|\theta_1-z_1\|^2+
    \frac{\eta}{2}\sum_{t=1}^T(\theta_t-z_t)^\top (P_t-P_{t-1})(\theta_t-z_t)
  \end{multline}

  Now we will see how the choices of $\eta$ enable the final sum from
  \eqref{eq:sum} to cancel the terms from \eqref{eq:functionBounds}.
  In  Case~\ref{it:exp}, we have that $\eta(P_t-P_{t-1}) \preceq \eta
  \nabla_t \nabla_t^\top$ and the bound from \eqref{eq:expBound} holds
  for $\rho=\eta$. In Case~\ref{it:strong}, $\eta(P_t-P_{t-1}) \preceq
  \eta H_t \preceq \ell I$. In Case~\ref{it:quad}, $\eta(P_t-P_{t-1})
  \preceq \eta H_t \preceq H_t$. Thus in all cases, $\eta$ has been
  chosen so that combining the appropriate term of
  (\ref{eq:functionBounds}) with \eqref{eq:sum} gives  
  \begin{equation}
  \begin{array}{ll}
    \label{eq:cancelled}
    \sum_{t=1}^T (f_t(\theta_t)-f_t(z_t)) &\le  \sum_{t=1}^T \frac{1}{2\eta}\nabla_t^\top P_t^{-1}
    \nabla_t  \\
    &\quad+ 2 D c_1\eta V + 2 \eta \epsilon D^2
  \end{array}
  \end{equation}

  Now we will bound the first sum of \eqref{eq:cancelled}. Note that
  $\nabla_t^\top P_t^{-1}\nabla_t = \langle P_t^{-1}
  ,\nabla_t\nabla_t^\top \rangle$. 
  In Case~\ref{it:exp}, we have that $\nabla_t \nabla_t^\top =
  P_t-\gamma P_{t-1}$, while in Cases \ref{it:strong} and
  \ref{it:quad}, we have that $\nabla_t \nabla_t^\top \preceq
  \frac{1}{\alpha} H_t = \frac{1}{\alpha}(P_t-\gamma P_{t-1})$.
  So, in Case~\ref{it:exp}, let $c_2 = 1$ and in Cases~\ref{it:strong} and
  \ref{it:quad}, let $c_2 = 1/\alpha$. Then in all cases, we have that
  \begin{equation}
    \label{eq:traceBound}
  \nabla_t^\top P_{t}^{-1} \nabla_t  \le c_2 \langle P_t^{-1},P_t-\gamma P_{t-1}\rangle.
  \end{equation}

  Lemma 4.5 of \cite{hazan2016introduction} shows that
  \begin{equation}
    \label{eq:logBound}
     \langle P_t^{-1},P_t-\gamma P_{t-1}\rangle \le \log
    \frac{|P_t|}{|\gamma P_{t-1}|} =\log
    \frac{|P_t|}{|P_{t-1}|}-n\log \gamma, 
  \end{equation}
  where $n$ is the dimension of $x_t$. 

  Combining (\ref{eq:traceBound}) with (\ref{eq:logBound}), summing,
  and then using the bound that $\|P_T\| \le c_1$
  gives,
  \begin{align}
    \nonumber
    \sum_{t=1}^T \nabla_t^\top P_t^{-1} \nabla_t &\le c_2 \log |P_T| -
                                                   c_2 n \log \epsilon -n T \log \gamma \\
    \label{eq:telescope}
    & \le c_2 n \log \frac{c_1}{\epsilon} -c_2nT \log \gamma
  \end{align}

    Recall that $c_1 = \epsilon + \frac{c_3}{1-\gamma}$, where $c_3 =
    G^2$ or $c_3 = u$, depending on the case. Then a more explicit upper bound on
    \eqref{eq:telescope} is given by: 
    \begin{equation}
      \label{eq:telescopeCrude}
      \sum_{t=1}^t \nabla_t^\top P_t^{-1} \nabla_t \le
      c_2 n \log\left(
        1 + \frac{c_3}{\epsilon(1-\gamma)}
        \right)
      - c_2 n T \log \gamma. 
\end{equation}
    
  Combining (\ref{eq:cancelled}) and (\ref{eq:telescopeCrude}) gives the bound:
  \begin{equation*}
  \begin{array}{l}
    \sum_{t=1}^T (f_t(\theta_t)-f_t(z_t)) \le
    -\frac{c_2 nT}{2\eta} \log \gamma + \\
    \frac{c_2 n}{2\eta}\log\left(1+\frac{c_3}{\epsilon(1-\gamma)}\right)+
    2D\eta\left(\epsilon + \frac{c_3}{1-\gamma}\right)V + 
    2\eta \epsilon D^2
  \end{array}
  \end{equation*}
    The desired regret bound can now be found by simplifying the
    expression on the right, using the fact that $\frac{1}{1-\gamma} >
    1$.  
\hfill\qed

The following integral bound will be used in a few places.

\begin{lemma}
  \label{lem:integral}
  If $\gamma \in (0,1)$, then
  \begin{equation*}
    \sum_{t=1}^T \frac{1}{1-\gamma^t}  \le \frac{1}{1-\gamma}+T-1 + \frac{\log(1-\gamma)}{\log\gamma}
  \end{equation*}
\end{lemma}
\begin{proof}
  \begin{align*}
\sum\limits_{t=1}^T\frac{1}{1-\gamma^t} &\le 
\frac{1}{1-\gamma}+\int_1^T\frac{1}{1-\gamma^t}\mathrm{d}t \\
                                        &=\frac{1}{1-\gamma}+\Big(t-\frac{\ln(1-\gamma^t)}{\ln(\gamma)}\Big)\Big|_1^T
                                          \\
&= \frac{1}{1-\gamma}+T-1 + \frac{\ln(1-\gamma)}{\ln\gamma} -
  \frac{\ln(1-\gamma^T)}{\ln\gamma} \\
                                        &\le \frac{1}{1-\gamma}+ T-1 + \frac{\ln(1-\gamma)}{\ln\gamma}.
\end{align*}
\end{proof}

\paragraph{Proof of Theorem~\ref{thm::quad_static_regret}:}
\begin{proof}

To proceed, recall that the update in Eq.\eqref{eq::quad_dis_update} is
\begin{equation*}
\begin{array}{ll}
\theta_{t+1} & = \frac{\gamma-\gamma^t}{1-\gamma^t}\theta_t + \frac{1-\gamma}{1-\gamma^t}y_t \\
             & = \theta_t - \eta_t \nabla f_t(\theta_t)
\end{array}
\end{equation*}
where $\eta_t = \frac{1-\gamma}{1-\gamma^t}$.

Then we get the relationship between $\nabla f_t(\theta_t)^T(\theta_t -\theta^*)$ and 
$\left\|\theta_t - \theta^*\right\|^2-\left\|\theta_{t+1}-\theta^*\right\|^2$ as:
\begin{equation*}
\begin{array}{ll}
\left\|\theta_{t+1}-\theta^*\right\|^2 
&= \left\|\theta_t-\eta_t\nabla f_t(\theta_t) - \theta^*\right\|^2 \\
&= \left\|\theta_t - \theta^*\right\|^2 -2\eta_t\nabla f_t(\theta_t)^T(\theta_t-\theta^*) \\
&\quad+ \eta_t^2\left\|\nabla f_t(\theta_t)\right\|^2
\end{array}
\end{equation*}
\begin{equation*}
\begin{array}{ll}
\nabla f_t(\theta_t)^T(\theta_t - \theta^*)
&= \frac{1}{2\eta_t}\big(\left\|\theta_t-\theta^*\right\|^2-\left\|\theta_{t+1}-\theta^*\right\|^2\big) \\
&\quad+\frac{\eta_t}{2}\left\|\nabla f_t(\theta_t)\right\|^2
\end{array}
\end{equation*}

Moreover, we write $f_t(\theta^*)$ as $f_t(\theta^*) = f_t(\theta_t)+\nabla f_t(\theta_t)^T(\theta^*-\theta_t)
+\frac{1}{2}\left\|\theta^*-\theta_t\right\|^2$,
which combined with the previous equation gives us the following equation:
\begin{equation*}
\begin{array}{ll}
f_t(\theta_t) - f_t(\theta^*) 
&= \frac{1}{2\eta_t}\big(\left\|\theta_t-\theta^*\right\|^2-\left\|\theta_{t+1}-\theta^*\right\|^2\big)\\
&\quad+\frac{\eta_t}{2}\left\|\nabla f_t(\theta_t)\right\|^2 - \frac{1}{2}\left\|\theta^*-\theta_t\right\|^2 \\
&\le 2D^2\eta_t + \frac{1}{2\eta_t}\big(\left\|\theta_t-\theta^*\right\|^2-\\
&\quad\left\|\theta_{t+1}-\theta^*\right\|^2\big)
- \frac{1}{2}\left\|\theta^*-\theta_t\right\|^2
\end{array}
\end{equation*}
where the inequality is due to 
$\left\|\nabla f_t(\theta_t)\right\|\le 2D$ as shown in Theorem \ref{thm::quad_dynamic_regret}.

Sum the above inequality from $t=1$ to $T$, we get:
\begin{equation*}
\begin{array}{ll}
\sum\limits_{t=1}^T\Big(f_t(\theta_t)-f_t(\theta^*)\Big) \\
\le 2D^2\sum\limits_{t=1}^T\eta_t 
+ \frac{1/\eta_1-1}{2}\left\|\theta_1-\theta^*\right\|^2 + \frac{1}{2}\sum\limits_{t=2}^T\big[(\frac{1}{\eta_t}\\
\quad-\frac{1}{\eta_{t-1}}-1)\left\|\theta^*-\theta_t\right\|^2\big] 
- \frac{1}{2\eta_T}\left\|\theta_{T+1}-\theta^*\right\|^2
\end{array}
\end{equation*}

Since $\eta_t = \frac{1-\gamma}{1-\gamma^t}$, $\eta_1 = 1$, $\frac{1}{\eta_t}-\frac{1}{\eta_{t-1}}-1<0$.
Then for the static regret, we have:
\begin{equation}
\begin{array}{ll}
\mathcal{R}_s = \sum\limits_{t=1}^T\Big(f_t(\theta_t)-f_t(\theta^*)\Big) \\
\le 2D^2\sum\limits_{t=1}^T\eta_t 
= 2D^2(1-\gamma)\sum\limits_{t=1}^T\frac{1}{1-\gamma^t}
\end{array}
\end{equation}


Now we will use the integral bound from Lemma~\ref{lem:integral} to
bound the regret.
Since $1-\gamma = 1/T^{\beta}$, $\frac{\log(1-\gamma)}{\log\gamma} = \frac{\beta\log T}{\log(1+\frac{1}{T^{\beta}-1})}$.
Since $\log(1+x)\ge \frac{1}{2}x, x\in(0,1)$, $\log(1+\frac{1}{T^{\beta}-1}) \ge \frac{1}{2}\frac{1}{T^{\beta}-1}$.
Thus, we have $\frac{\log(1-\gamma)}{\log\gamma}\le 2\beta (T^{\beta}-1)\log T$.
Then $(1-\gamma)\sum\limits_{t=1}^T\frac{1}{1-\gamma^t} = O(T^{1-\beta})$, 
which results in
$\mathcal{R}_s \le O(T^{1-\beta})$.
\end{proof}

\paragraph{Proof of Lemma~\ref{lem:gen_prob_var_path}:}
\begin{proof}

The proof follows the analysis in Chapter 2 of \cite{nesterov2013introductory}.

From the strong convexity of $f_t(\theta)$, we have 
\small
\begin{equation}
\label{eq::gen_strongly_ineq}
\begin{array}{ll}
f_t(\theta) &\ge f_t(\theta_t)+\nabla f_t(\theta_t)^T(\theta-\theta_t)+\frac{\ell}{2}\left\|\theta-\theta_t\right\|^2\\
&= f_t(\theta_t)+\nabla f_t(\theta_t)^T(\theta-\theta_t) +\nabla f_t(\theta_t)^T(\theta_{t+1}-\theta_t) \\
&\quad-\nabla f_t(\theta_t)^T(\theta_{t+1}-\theta_t)+\frac{\ell}{2}\left\|\theta-\theta_t\right\|^2 \\
&= f_t(\theta_t) +\nabla f_t(\theta_t)^T(\theta_{t+1}-\theta_t) \\
&\quad+\nabla f_t(\theta_t)^T(\theta-\theta_{t+1})+\frac{\ell}{2}\left\|\theta-\theta_t\right\|^2
\end{array}
\end{equation}
\normalsize

According to the optimality condition of the update rule in Eq.\eqref{eq::gen_prob_update},
we have $\big(\nabla f_t(\theta_t)+\frac{1}{\eta_t}(\theta_{t+1}-\theta_t)\big)^T(\theta-\theta_{t+1})\ge 0,\forall \theta\in\cS$,
which is $\nabla f_t(\theta_t)^T(\theta-\theta_{t+1})\ge \frac{1}{\eta_t}(\theta_{t}-\theta_{t+1})^T(\theta-\theta_{t+1})$.
Then combine with Eq.\eqref{eq::gen_strongly_ineq}, we have 
\begin{equation}
\label{eq::gen_strongly_final}
\begin{array}{ll}
f_t(\theta) &\ge f_t(\theta_t) +\nabla f_t(\theta_t)^T(\theta_{t+1}-\theta_t) \\
&\quad+\frac{1}{\eta_t}(\theta_{t}-\theta_{t+1})^T(\theta-\theta_{t+1}) +\frac{\ell}{2}\left\|\theta-\theta_t\right\|^2
\end{array}
\end{equation}

From the smoothness of $f_t(\theta)$, we have 
$f_t(\theta_{t+1}) \le f_t(\theta_t)+\nabla f_t(\theta_t)^T(\theta_{t+1}-\theta_t)+\frac{u}{2}\left\|\theta_{t+1}-\theta_t\right\|^2$.
Since $\frac{1}{\eta_t} = \frac{\ell(\gamma-\gamma^t)+u(1-\gamma)}{1-\gamma}\ge u$,
we have $f_t(\theta_t)+\nabla f_t(\theta_t)^T(\theta_{t+1}-\theta_t) 
\ge f_t(\theta_{t+1}) - \frac{1}{2\eta_t}\left\|\theta_{t+1}-\theta_t\right\|^2$.
Then combined with inequality \eqref{eq::gen_strongly_final}, we have 
\begin{equation}
\begin{array}{ll}
f_t(\theta) &\ge f_t(\theta_{t+1})-\frac{1}{2\eta_t}\left\|\theta_{t+1}-\theta_t\right\|^2 \\
&\quad+ \frac{1}{\eta_t}(\theta_{t}-\theta_{t+1})^T(\theta-\theta_{t+1})
+\frac{\ell}{2}\left\|\theta-\theta_t\right\|^2\\
& = f_t(\theta_{t+1})+\frac{1}{2\eta_t}\left\|\theta_{t+1}-\theta_t\right\|^2 \\
&\quad+ \frac{1}{\eta_t}(\theta_{t}-\theta_{t+1})^T(\theta-\theta_{t})
+\frac{\ell}{2}\left\|\theta-\theta_t\right\|^2
\end{array}
\end{equation}

By setting $\theta = \theta_t^*$ and using the fact $f_t(\theta_t^*)\le f_t(\theta_{t+1})$,
we reformulate the above inequality as:
\begin{equation}
\begin{array}{l}
(\theta_{t}-\theta_{t+1})^T(\theta_t^*-\theta_{t}) \\
\le -\frac{\ell(1-\gamma)}{2\ell(\gamma-\gamma^t)+2u(1-\gamma)}\left\|\theta_t^*-\theta_t\right\|^2
-\frac{1}{2}\left\|\theta_{t+1}-\theta_t\right\|^2
\end{array}
\end{equation}

Since $\left\|\theta_{t+1} -\theta_t^*\right\|^2 = \left\|\theta_{t+1}-\theta_t+\theta_t -\theta_t^*\right\|^2$,
we have
\begin{equation}
\begin{array}{ll}
\left\|\theta_{t+1} -\theta_t^*\right\|^2 &= \left\|\theta_{t+1} -\theta_t\right\|^2  + \left\|\theta_{t} -\theta_t^*\right\|^2 \\
&\quad+ 2(\theta_{t}-\theta_{t+1})^T(\theta_t^*-\theta_{t}) \\
& \le \big(1-\frac{\ell(1-\gamma)}{\ell(\gamma-\gamma^t)+u(1-\gamma)}\big)\left\|\theta_{t} -\theta_t^*\right\|^2 \\
& \le \big(1-\frac{\ell(1-\gamma)}{\ell\gamma+u(1-\gamma)}\big)\left\|\theta_{t} -\theta_t^*\right\|^2
\end{array}
\end{equation}

\end{proof}

\paragraph{Proof of Theorem~\ref{thm::gen_prob_dynamic_regret}:}
\begin{proof}

We use the same steps as in the previous section.
First, according to the Mean Value Theorem, 
we have 
$f_t(\theta_t)-f_t(\theta_t^*) = \nabla f_t(x)^T(\theta_t-\theta_t^*)
\le \left\|\nabla f_t(x)\right\|\left\|\theta_t-\theta_t^*\right\|$,
where $x\in \{v| v = \delta \theta_t + (1-\delta)\theta_t^*,\delta\in[0,1]\}$.
Due to the assumption on the upper bound of the norm of the gradient, we have 
$f_t(\theta_t)-f_t(\theta_t^*) \le G\left\|\theta_t-\theta_t^*\right\|$.
As a result, 
$\sum\limits_{t=1}^T\big(f_t(\theta_t)-f_t(\theta_t^*)\big)\le G\sum\limits_{t=1}^T\left\|\theta_t-\theta_t^*\right\|$.

Now we need to upper bound the term $\sum\limits_{t=1}^T\left\|\theta_t-\theta_t^*\right\|$.
$\sum\limits_{t=1}^T\left\|\theta_t-\theta_t^*\right\|$ is equal to $ \left\|\theta_1-\theta_1^*\right\| 
+ \sum\limits_{t=2}^T\left\|\theta_t-\theta_{t-1}^*+\theta_{t-1}^*-\theta_t^*\right\|$,
which is less than 
$\left\|\theta_1-\theta_1^*\right\| + \sum\limits_{t=1}^{T}\left\|\theta_{t+1}-\theta_{t}^*\right\| 
+ \sum\limits_{t=2}^T\left\|\theta_t^*-\theta_{t-1}^*\right\|$.
According to Lemma \ref{lem:gen_prob_var_path},
we have $\sum\limits_{t=1}^{T}\left\|\theta_{t+1}-\theta_{t}^*\right\| 
\le \rho\sum\limits_{t=1}^T\left\|\theta_{t}-\theta_{t}^*\right\|$ 
with $\rho = \sqrt{1-\frac{l(1-\gamma)}{u(1-\gamma)+l\gamma}} $.
Then we have 
$\sum\limits_{t=1}^T\left\|\theta_t-\theta_t^*\right\| \le
\left\|\theta_1-\theta_1^*\right\| + \rho\sum\limits_{t=1}^{T}\left\|\theta_{t}-\theta_{t}^*\right\| 
+ \sum\limits_{t=2}^T\left\|\theta_t^*-\theta_{t-1}^*\right\|$,
which can be reformulated as 
$\sum\limits_{t=1}^T\left\|\theta_t-\theta_t^*\right\| \le
\frac{1}{1-\rho}(\left\|\theta_1-\theta_1^*\right\| +
+ \sum\limits_{t=2}^T\left\|\theta_t^*-\theta_{t-1}^*\right\|)$.

$1-\rho = 1-\sqrt{1-\frac{a_0}{b_0}}= \frac{\sqrt{b_0}-\sqrt{b_0-a_0}}{\sqrt{b_0}}$, 
where $a_0 = \ell$ and $b_0 = \frac{\ell\gamma+u(1-\gamma)}{1-\gamma}$.
Thus, $1/(1-\rho) = \frac{\sqrt{b_0}}{\sqrt{b_0}-\sqrt{b_0-a_0}}
= \frac{\sqrt{b_0}(\sqrt{b_0}+\sqrt{b_0-a_0})}{a_0}$.
After plugging in the expression of $1-\gamma = 1/T^{\beta}$,
$1/(1-\rho) = \frac{\sqrt{\ell(T^{\beta}-1)+u}\big(\sqrt{\ell(T^{\beta}-1)+u}+\sqrt{\ell(T^{\beta}-1)+u-\ell}\big)}{\ell}
\le \frac{2\big(\ell(T^{\beta}-1)+u\big)}{\ell} = 2(T^{\beta}-1)+u/\ell$

Then $\mathcal{R}_d = \sum\limits_{t=1}^T\big(f_t(\theta_t)-f_t(\theta_t^*)\big)
\le G\frac{1}{1-\rho}\big(\left\|\theta_1-\theta_1^*\right\| +
+ \sum\limits_{t=2}^T\left\|\theta_t^*-\theta_{t-1}^*\right\|\big)
\le G\big(2(T^{\beta}-1)+u/\ell\big)\big(\left\|\theta_1-\theta_1^*\right\| +
+ \sum\limits_{t=2}^T\left\|\theta_t^*-\theta_{t-1}^*\right\|\big)$.

\end{proof}

\paragraph{Proof of Theorem~\ref{thm::gen_prob_static_regret}:}
\begin{proof}

The proof follows the similar steps in the proof of Theorem \ref{thm::quad_static_regret}.

According to the non-expansive property of the projection operator and the update rule in Eq.\eqref{eq::gen_prob_update},
we have
\begin{equation*}
\begin{array}{ll}
\left\|\theta_{t+1}-\theta^*\right\|^2& \le \left\|\theta_t-\eta_t\nabla f_t(\theta_t)-\theta^*\right\|^2 \\
& = \left\|\theta_t-\theta^*\right\|^2 -2\eta_t\nabla f_t(\theta_t)^T(\theta_t-\theta^*)\\
&\quad+\eta_t^2\left\|\nabla f_t(\theta_t)\right\|^2
\end{array}
\end{equation*}
The reformulation gives us
\begin{equation} 
\label{eq::gen_prob_contraction_ineq}
\begin{array}{ll}
\nabla f_t(\theta_t)^T(\theta_t-\theta^*)
&\le \frac{1}{2\eta_t}\big(\left\|\theta_t-\theta^*\right\|^2 - \left\|\theta_{t+1}-\theta^*\right\|^2\big)\\
&\quad+ \frac{\eta_t}{2}\left\|\nabla f_t(\theta_t)\right\|^2
\end{array}
\end{equation}

Moreover, from the strong convexity, we have 
$f_t(\theta^*)\ge f_t(\theta_t)+\nabla f_t(\theta_t)^T(\theta^*-\theta_t)+\frac{\ell}{2}\left\|\theta^*-\theta_t\right\|^2$,
which is equivalent to 
$\nabla f_t(\theta_t)^T(\theta_t-\theta^*)\ge f_t(\theta_t)-f_t(\theta^*)+\frac{\ell}{2}\left\|\theta^*-\theta_t\right\|^2$.
Combined with Eq.\eqref{eq::gen_prob_contraction_ineq}, we have
\begin{equation*}
\begin{array}{ll}
f_t(\theta_t)-f_t(\theta^*) &\le \frac{1}{2\eta_t}\big(\left\|\theta_t-\theta^*\right\|^2 
- \left\|\theta_{t+1}-\theta^*\right\|^2\big)\\
&\quad+ \frac{\eta_t}{2}\left\|\nabla f_t(\theta_t)\right\|^2-\frac{\ell}{2}\left\|\theta^*-\theta_t\right\|^2
\end{array}
\end{equation*}

Summing up from $t=1$ to $T$ with $\left\|\nabla f_t(\theta_t)\right\|^2\le G^2$, we get 
\small
\begin{equation}
\label{eq::gen_prob_static_final}
\begin{array}{ll}
\sum\limits_{t=1}^T\big(f_t(\theta_t)-f_t(\theta^*)\big)\\
\le \sum\limits_{t=1}^T\frac{1}{2\eta_t}\big(\left\|\theta_t-\theta^*\right\|^2 - \left\|\theta_{t+1}-\theta^*\right\|^2\big)\\
\quad + \sum\limits_{t=1}^T\frac{\eta_t}{2}G^2-\sum\limits_{t=1}^T\frac{\ell}{2}\left\|\theta^*-\theta_t\right\|^2 \\
\le G^2/2\sum\limits_{t=1}^T\eta_t + \frac{1/\eta_1-\ell}{2}\left\|\theta_1-\theta^*\right\|^2 \\
\quad+ \frac{1}{2}\sum\limits_{t=2}^T\Big[(\frac{1}{\eta_t}-\frac{1}{\eta_{t-1}}-\ell)\left\|\theta^*-\theta_t\right\|^2\Big] 
\end{array}
\end{equation}
\normalsize
Since $\eta_t = \frac{1-\gamma}{\ell(\gamma-\gamma^t)+u(1-\gamma)}$, $1/\eta_1 = u$
and $\frac{1}{\eta_t}-\frac{1}{\eta_{t-1}}-\ell = \frac{\ell(\gamma^{t-1}-1)(1-\gamma)}{1-\gamma}\le 0$.

For the term $\sum\limits_{t=1}^T\eta_t = \sum\limits_{t=1}^T\frac{1-\gamma}{\ell(\gamma-\gamma^t)+u(1-\gamma)}$,
it can be reformulated as
$\frac{1}{u}\sum\limits_{t=1}^T\frac{\frac{u(1-\gamma)}{\ell(\gamma-\gamma^t)}}{1+\frac{u(1-\gamma)}{\ell(\gamma-\gamma^t)}}
= \frac{1}{u} + \frac{1}{u}\sum\limits_{t=2}^T\frac{\frac{u(1-\gamma)}{\ell(\gamma-\gamma^t)}}{1+\frac{u(1-\gamma)}{\ell(\gamma-\gamma^t)}}
\le \frac{1}{u}+\frac{1}{u}\sum\limits_{t=2}^T\frac{u(1-\gamma)}{\ell(\gamma-\gamma^t)}
= \frac{1}{u} + \frac{1-\gamma}{\ell\gamma}\sum\limits_{t=2}^T\frac{1}{1-\gamma^{t-1}}
= \frac{1}{u} + \frac{1-\gamma}{\ell\gamma}\sum\limits_{t=1}^{T-1}\frac{1}{1-\gamma^{t}}$.
For $\sum\limits_{t=1}^{T-1}\frac{1}{1-\gamma^{t}}$, 
we know that $\sum\limits_{t=1}^{T-1}\frac{1}{1-\gamma^{t}} \le O(T)$ as shown in the proof of Theorem \ref{thm::quad_static_regret}.
For the term $\frac{1-\gamma}{\ell\gamma}$, $\frac{1-\gamma}{\ell\gamma} = \frac{1}{\ell(T^{\beta}-1)}$.
Combining these two terms' inequalities, 
we get that $\sum\limits_{t=1}^T\eta_t \le O(T^{1-\beta})$.

As a result, the inequality \eqref{eq::gen_prob_static_final} can be reduced to
\begin{equation*}
\sum\limits_{t=1}^T\big(f_t(\theta_t)-f_t(\theta^*)\big) \le O(T^{1-\beta})
\end{equation*}

\end{proof}

\paragraph{Proof of Corollary~\ref{corr:strongly_dynamic_regret}:}
\begin{proof}
Since $\gamma = 1-\frac{1}{2}\sqrt{\frac{\max\{V,\log^2 T/T\}}{2DT}}$ and $V\in[0,2DT]$,
$1/2\le\gamma<1$. 

Next, we upper bound each term on the right-hand-side of Theorem \ref{thm::strongly_regret} individually. 
$\frac{1}{1-\gamma}V=2\sqrt{\frac{2DT}{\max\{V,\log^2 T/T\}}}V\le O(\sqrt{TV})$.
In order to bound the second term, 
Lemma~\ref{lem:integral} implies that
$(1-\gamma)\sum\limits_{t=1}^T\frac{1}{1-\gamma^t}\le 1+(1-\gamma)(T +
\frac{\ln(1-\gamma)}{\ln\gamma})$.

  In this case, the logarithm terms can be bounded by:
\begin{equation*}
\begin{array}{l}
\frac{\ln(1-\gamma)}{\ln\gamma} \\
= \frac{-\ln (\frac{1}{2}\sqrt{\frac{\max\{V,\log^2 T/T\}}{2DT}})}{-\ln (1-\frac{1}{2}\sqrt{\frac{\max\{V,\log^2 T/T\}}{2DT}})}\\
= \frac{-\ln (\frac{1}{2}\sqrt{\frac{\max\{V,\log^2 T/T\}}
{2DT}})}{\ln\Big(1+\frac{\frac{1}{2}\sqrt{\frac{\max\{V,\log^2 T/T\}}{2DT}}}{1-\frac{1}{2}\sqrt{\frac{\max\{V,\log^2 T/T\}}{2DT}}}\Big)}\\
\le \ln (2\sqrt{\frac{2DT}{\max\{V,\log^2 T/T\}}})4\sqrt{\frac{2DT}{\max\{V,\log^2 T/T\}}}\\
\le O(\ln(T/\log T)\frac{T}{\log T} ) \\
\le O(T) 
\end{array}
\end{equation*}
where the first inequality follows by using $\ln(1+x)\ge \frac{1}{2}x, x\in[0,1]$, 
and $1-\frac{1}{2}\sqrt{\frac{\max\{V,\log^2 T/T\}}{2DT}}<1$.

Thus, $(1-\gamma)\sum\limits_{t=1}^T\frac{1}{1-\gamma^t}\le \max\{O(\log T),O(\sqrt{TV})\}$.
The final result follows by adding the two terms.
\end{proof}

\paragraph{Proof of Lemma~\ref{lem:meta-expert-compare}:}
\begin{proof}
The first part of the proof is the same as the first part of the result in the Proof of Lemma 1 in \cite{zhang2018adaptive},
which follows methods of \cite{cesa2006prediction}.
We define $L_t^\gamma = \sum\limits_{i=1}^tf_i(\theta_i^\gamma)$, 
and $W_t = \sum\limits_{\gamma\in\cH}w_1^\gamma \exp(-\alpha L_t^\gamma)$.

The following update is equivalent to the update rule in Algorithm \ref{alg:meta}:
\begin{equation}
\label{eq:expert_reform}
w_t^\gamma = \frac{w_1^\gamma \exp(-\alpha L_{t-1}^\gamma)}
              {\sum\limits_{\mu\in\cH} w_1^\mu \exp(-\alpha L_{t-1}^\mu)}, \quad t\ge 2.
\end{equation} 

First, we have 
\begin{equation}
\label{eq:logW_lower}
\begin{array}{ll}
\log W_T &= \log\big(\sum\limits_{\gamma\in\cH}w_1^\gamma\exp(-\alpha L_T^\gamma)\big) \\
&\ge \log\big(\max\limits_{\gamma\in\cH}w_1^\gamma\exp(-\alpha L_T^\gamma)\big) \\
&=-\alpha \min\limits_{\gamma\in\cH}\big(L_T^\gamma+\frac{1}{\alpha}\log\frac{1}{w_1^\gamma}\big).
\end{array}
\end{equation}

Then we bound the quantity $\log(W_t/W_{t-1})$. 
For $t\ge2$, we get
\begin{equation}
\begin{array}{l}
\log\Big(\frac{W_t}{W_{t-1}}\Big) \\
= \log\Big(\frac{\sum_{\gamma\in\cH}w_1^\gamma\exp(-\alpha L_t^\gamma)}
        {\sum_{\gamma\in\cH}w_1^\gamma\exp(-\alpha L_{t-1}^\gamma)}\Big)\\
= \log\Big(\frac{\sum_{\gamma\in\cH}w_1^\gamma\exp(-\alpha L_{t-1}^\gamma)\exp(-\alpha f_t(\theta_t^\gamma))}
        {\sum_{\gamma\in\cH}w_1^\gamma\exp(-\alpha L_{t-1}^\gamma)}\Big)\\
=\log\Big(\sum\limits_{\gamma\in\cH}w_t^\gamma\exp(-\alpha f_t(\theta_t^\gamma))\Big)

\end{array}
\end{equation}
where the last equality is due to Eq.\eqref{eq:expert_reform}.

When $t=1$, $\log W_1 = \log\Big(\sum\limits_{\gamma\in\cH}w_1^\gamma\exp(-\alpha f_1(\theta_1^\gamma))\Big)$.
Then $\log W_T$ can be expressed as:
\begin{equation}
\label{eq::logW_upper}
\begin{array}{ll}
\log W_T &= \log W_1 + \sum\limits_{t=2}^T\log\Big(\frac{W_t}{W_{t-1}}\Big) \\
&= \sum\limits_{t=1}^T\log\Big(\sum\limits_{\gamma\in\cH}w_t^\gamma\exp(-\alpha f_t(\theta_t^\gamma))\Big).
\end{array}
\end{equation}

The rest of the proof is new.

Due to the $\alpha$-exp-concavity,
$\exp(-\alpha f_t(\sum_{\gamma\in\cH}w_t^\gamma \theta_t^\gamma))
\ge \sum_{\gamma\in\cH}w_t^\gamma \exp(-\alpha f_t(\theta_t^\gamma))$,
which is equivalent to 
\small
\begin{equation}
\label{eq::exp_concave_ineq}
\begin{array}{ll}
\log\Big(\sum_{\gamma\in\cH}w_t^\gamma \exp(-\alpha f_t(\theta_t^\gamma))\Big) 
&\le -\alpha f_t\Big(\sum_{\gamma\in\cH}w_t^\gamma \theta_t^\gamma\Big) \\
&=-\alpha f_t(\theta_t)
\end{array}
\end{equation}
\normalsize

Combining the Inequalities \eqref{eq:logW_lower}, \eqref{eq::logW_upper}, and \eqref{eq::exp_concave_ineq},
we get 
\begin{equation*}
-\alpha \min\limits_{\gamma\in\cH}\big(L_T^\gamma+\frac{1}{\alpha}\log\frac{1}{w_1^\gamma}\big) 
\le -\alpha\sum_{t=1}^T f_t(\theta_t)
\end{equation*}
which can be reformulated as 
\begin{equation*}
\sum_{t=1}^T f_t(\theta_t)\le\min\limits_{\gamma\in\cH}\Big(\sum_{t=1}^Tf_t(\theta_t^\gamma)+\frac{1}{\alpha}\log\frac{1}{w_1^\gamma}\Big)
\end{equation*}

Since it holds for the minimum value, it is true for all $\gamma\in\cH$,
which completes the proof.

\end{proof}

\paragraph{Proof of Theorem~\ref{thm:mega_exp-concave}:}
\begin{proof}

When $\gamma = \gamma^* = 
1-\frac{1}{2}\frac{\log T}{T}\sqrt{\frac{\max\{\frac{T}{\log^2 T}V,1\}}{2D}}
=1-\eta^*$, we have $\sum_{t=1}^T (f_t(\theta_t^{\gamma^*})-f_t(z_t))\le \max\{O(\log T), O(\sqrt{TV})\}$
based on the Corollary \ref{cor:logBounds}.

Since $0\le V\le 2TD$, $\frac{1}{2}\frac{\log T}{T\sqrt{2D}}\le\eta^*\le \frac{1}{2}$.

According to our definition of $\eta_i$, $\min \eta_i = \frac{1}{2}\frac{\log T}{T\sqrt{2D}}$
and $\frac{1}{2}\le \max \eta_i< 1$, 
which means for any value of $V$, there always exists a $\eta_k$ such that
\begin{equation*}
\eta_k = \frac{1}{2}\frac{\log T}{T\sqrt{2D}}2^{k-1}\le \eta^*\le 2\eta_k = \eta_{k+1}
\end{equation*}
where $k = \lfloor \frac{1}{2}\log_2 (\max\{\frac{T}{\log^2 T}V,1\})\rfloor+1$.

Now we claim that that running the algorithm with $\gamma_k$ incurs at
most a constant factor increase in the dynamic regret. 

Since $0<\eta_k\le \frac{1}{2}$, $\frac{1}{2}\le\gamma_k = 1-\eta_k<1$ and $\gamma_k\ge\gamma^*$.

According to Theorem \ref{thm:expConcaveThm}, we have 
\small
  \begin{equation*}
  \begin{array}{ll}
    \sum_{t=1}^T (f_t(\theta_t^{\gamma_k})-f_t(z_t)) &\le -a_1 T \log \gamma_k -a_2\log(1-\gamma_k)\\
     &\quad+ \frac{a_3}{1-\gamma_k} V + a_4.
   \end{array}
  \end{equation*}
  \normalsize

  Now we bound each term of the regret in terms of the value obtained
  by using $\gamma^*$.
For the first term on the RHS, $-T\log\gamma_k = T\log\frac{1}{\gamma_k}\le T\log\frac{1}{\gamma^*}$.

For the second one, $-\log(1-\gamma_k) = -\log\frac{1}{2}(2-2\gamma_k) = -\log\frac{1}{2}2\eta_k$.
Since $1\ge2\eta_k\ge\eta^*$, $\frac{1}{2}2\eta_k\ge\frac{1}{2}\eta^*$,
which leads to $-\log\frac{1}{2}2\eta_k\le -\log\frac{1}{2}\eta^*$
and $-\log(1-\gamma_k)\le -\log\frac{1}{2}\eta^* = \log2 -\log(1-\gamma^*)$.

For the third one, $\frac{1}{1-\gamma_k} V 
= \frac{1}{\eta_k}V = \frac{2}{2\eta_k}V\le \frac{2}{\eta^*}V
=\frac{2}{1-\gamma^*}V$.
Thus the claim has been proved.

Since using $\gamma_k$ in place of $\gamma^*$ increases the
regret by at most a constant factor, Corollary \ref{cor:logBounds}
implies that:
\begin{equation}
\label{eq:expert_k_regret}
\sum_{t=1}^T (f_t(\theta_t^{\gamma_k})-f_t(z_t))\le \max\{O(\log T), O(\sqrt{TV})\}
\end{equation}

Furthermore, from Lemma \ref{lem:meta-expert-compare} we get 
\begin{equation}
\label{eq:expert_k_comparable}
\begin{array}{ll}
\sum\limits_{t=1}^T (f_t(\theta_t) - f_t(\theta_t^{\gamma_k}))
&\le \frac{1}{\alpha}\log\frac{1}{w_1^{\gamma_k}}\\
&\le \frac{1}{\alpha}\log(k(k+1)) \\
&\le 2\frac{1}{\alpha}\log(k+1) \\
&\le O(\log(\log T))
\end{array}
\end{equation}

Combining the above inequalities \eqref{eq:expert_k_regret} and \eqref{eq:expert_k_comparable} completes the proof.
\end{proof}

\paragraph{Proof of Lemma~\ref{lem::strongly_is_exp}:}
\begin{proof}
Let $g(x) = \exp(-\alpha f(x))$. To prove the concavity of $g(x)$,
it is equivalent to show $\langle\nabla g(x)-\nabla g(y),x-y\rangle\le 0,x,y\in\cS$.
Since $\nabla g(x) = \exp(-\alpha f(x))(-\alpha)\nabla f(x)$,
it is equivalent to prove that 
$\langle \exp(-\alpha f(x))\nabla f(x)-\exp(-\alpha f(y))\nabla f(y),x-y\rangle\ge 0$,
which can be reformulated as
\small
\begin{equation}
\label{eq:strong_is_exp_exp}
\exp(-\alpha f(x))\langle \nabla f(x),x-y\rangle \ge \exp(-\alpha f(y))\langle \nabla f(y),x-y\rangle
\end{equation}
\normalsize

Without loss of generality, let us assume $f(x)\ge f(y)$.
Due to $\ell$-strong convexity, $f(x)\ge f(y) + \langle \nabla f(y),x-y\rangle+\frac{\ell}{2}\|x-y\|^2$,
which leads to 
\begin{equation}
\label{eq:strong_is_exp_p1}
\langle \nabla f(y),x-y\rangle \le f(x)-f(y)-\frac{\ell}{2}\|x-y\|^2
\end{equation}

What's more, $f(y)\ge f(x) + \langle \nabla f(x),y-x\rangle+\frac{\ell}{2}\|x-y\|^2$,
which leads to
\begin{equation}
\label{eq:strong_is_exp_p2}
\langle \nabla f(x),x-y\rangle \ge f(x)-f(y)+\frac{\ell}{2}\|x-y\|^2
\end{equation}

Combining inequalities \eqref{eq:strong_is_exp_exp}, \eqref{eq:strong_is_exp_p1}, and \eqref{eq:strong_is_exp_p2},
it is enough to prove that
$\exp(-\alpha f(x))(f(x)-f(y)+\frac{\ell}{2}\|x-y\|^2)\ge \exp(-\alpha f(y))(f(x)-f(y)-\frac{\ell}{2}\|x-y\|^2)$,
which can be reformulated as
$\frac{\ell}{2}\|x-y\|^2(\exp(-\alpha f(x))+\exp(-\alpha f(y)))
\ge (f(x)-f(y))(\exp(-\alpha f(y))-\exp(-\alpha f(x)))$.
When $x-y = 0$, it is always true. Let us consider the case when $\|x-y\|>0$.
Then we need to show that
$\frac{\ell}{2}\Big(1+\exp\big(\alpha\big(f(x)-f(y)\big)\big)\Big)
\ge \frac{f(x)-f(y)}{\|x-y\|}\frac{\exp\Big(\alpha\big(f(x)-f(y)\big)\Big)-1}{\|x-y\|}$.
Due to bounded gradient and Mean value theorem,$\frac{f(x)-f(y)}{\|x-y\|}\le G$,
which means it is enough to show that
\small
\begin{equation}
\label{eq:strong_is_exp_p3}
\frac{\ell}{2G}\Big(1+\exp\big(\alpha\big(f(x)-f(y)\big)\big)\Big)
\ge\frac{\exp\Big(\alpha\big(f(x)-f(y)\big)\Big)-1}{\|x-y\|}
\end{equation}
\normalsize

According to the Taylor series, 
$\exp\Big(\alpha\big(f(x)-f(y)\big)\Big) 
= 1 + \alpha\big(f(x)-f(y)\big)+\frac{1}{2!}\alpha^2\big(f(x)-f(y)\big)^2
+\dots+\frac{1}{n!}\alpha^n\big(f(x)-f(y)\big)^n,n\to \infty$.
Thus, $\frac{\exp\Big(\alpha\big(f(x)-f(y)\big)\Big)-1}{\|x-y\|}
 = \alpha \frac{f(x)-f(y)}{\|x-y\|}+\frac{1}{2}\alpha^2(f(x)-f(y))\frac{f(x)-f(y)}{\|x-y\|}
 +\dots+\frac{1}{n!}\alpha^n\big(f(x)-f(y)\big)^{n-1}\frac{f(x)-f(y)}{\|x-y\|},n\to \infty$.
 Since $\frac{f(x)-f(y)}{\|x-y\|}\le G$, we have
 \begin{equation}
 \label{eq:strong_is_exp_f1}
 \begin{array}{l}
 \frac{\exp\Big(\alpha\big(f(x)-f(y)\big)\Big)-1}{\|x-y\|}\\
 \le \alpha G+\frac{1}{2}\alpha^2(f(x)-f(y)) G +\dots\\
 \quad+\frac{1}{n!}\alpha^n\big(f(x)-f(y)\big)^{n-1}G
 \end{array}
 \end{equation}

For the LHS of inequality \eqref{eq:strong_is_exp_p3}, it is equal to
\begin{equation}
\label{eq:strong_is_exp_f2}
\begin{array}{l}
\frac{\ell}{G}+\alpha\frac{\ell}{2G}(f(x)-f(y))
+\frac{1}{2!}\alpha^2\frac{\ell}{2G}(f(x)-f(y))^2\\
+\dots
+\frac{1}{n!}\alpha^n\frac{\ell}{2G}(f(x)-f(y))^n,n\to \infty
\end{array}
\end{equation}

If we compare the coefficients of the RHS from the inequality \eqref{eq:strong_is_exp_f1} with the one in \eqref{eq:strong_is_exp_f2}
and plug in $\alpha = \ell/G^2$, we see that it is always smaller or equal,
which completes the proof.

\end{proof}

\paragraph{Proof of Theorem~\ref{thm:mega_strongly}:}
\begin{proof}
As in the proof of Theorem \ref{thm:mega_exp-concave}, all we need to show is that
there exists an algorithm $A^\gamma$, 
which can bound the regret $\sum_{t=1}^T (f_t(\theta_t^\gamma)-f_t(z_t)) 
\le O(\max\{\log T, \sqrt{TV}\})$.

When $\gamma = \gamma^* = 
1-\frac{1}{2}\frac{\log T}{T}\sqrt{\frac{\max\{\frac{T}{\log^2 T}V,1\}}{2D}}
=1-\eta^*$, we have $\sum_{t=1}^T (f_t(\theta_t^{\gamma^*})-f_t(z_t))\le O(\max\{\log T, \sqrt{TV}\})$
based on the Corollary \ref{corr:strongly_dynamic_regret}.

Since $0\le V\le 2TD$, $\frac{1}{2}\frac{\log T}{T\sqrt{2D}}\le\eta^*\le \frac{1}{2}$.

According to our definition of $\eta_i$, $\min \eta_i = \frac{1}{2}\frac{\log T}{T\sqrt{2D}}$
and $\frac{1}{2}\le \max \eta_i< 1$, 
which means for any value of $V$, there always exists a $\eta_k$ such that
\begin{equation*}
\eta_k = \frac{1}{2}\frac{\log T}{T\sqrt{2D}}2^{k-1}\le \eta^*\le 2\eta_k = \eta_{k+1}
\end{equation*}
where $k = \lfloor \frac{1}{2}\log_2 (\max\{\frac{T}{\log^2 T}V,1\})\rfloor+1$.

Since $0<\eta_k\le \frac{1}{2}$, $\frac{1}{2}\le\gamma_k = 1-\eta_k<1$ and $\gamma_k\ge\gamma^*$.

According to Theorem \ref{thm::strongly_regret}, we have 
\begin{equation*}
\sum\limits_{t=1}^T \big(f_t(\theta_t^{\gamma_k})-f_t(z_t)\big) 
\le  \frac{2D\ell}{1-\gamma_k}V + \frac{G^2}{\ell}(1-\gamma_k)\sum\limits_{t=1}^T\frac{1}{1-\gamma_k^t}
\end{equation*}

For the first term on the RHS, $\frac{1}{1-\gamma_k} V
= \frac{1}{\eta_k}V = \frac{2}{2\eta_k}V\le \frac{2}{\eta^*}V
=\frac{2}{1-\gamma^*}V$.

For the second one, $1-\gamma_k\le 1-\gamma^*$. 
According to the proof in Corollary \ref{corr:strongly_dynamic_regret},
$\sum\limits_{t=1}^T\frac{1}{1-\gamma_k^t} \le \frac{1}{1-\gamma_k}+T + \frac{\log(1-\gamma_k)}{\log\gamma_k}$.
\begin{equation}
\label{eq:coef_strongly1}
\frac{\log(1-\gamma_k)}{\log\gamma_k} = \frac{\log \eta_k}{\log (1-\eta_k)}
= \frac{-\log \eta_k}{-\log (1-\eta_k)}. 
\end{equation}
Since $\eta_k\ge \frac{1}{2}\eta^*$,
$\log\eta_k\ge\log\frac{1}{2}\eta^*$ 
and 
\begin{equation}
\label{eq:coef_strongly2}
0<-\log\eta_k\le-\log\frac{1}{2}\eta^* = \log 2-\log \eta^*.
\end{equation}
Since $\eta_k\ge \frac{1}{2}\eta^*$, $1-\eta_k\le 1-\frac{1}{2}\eta^*$.
Then $\log(1-\eta_k)\le \log(1-\frac{1}{2}\eta^*)$,
which results in 
\begin{equation}
\label{eq:coef_strongly3}
-\log(1-\eta_k)\ge -\log(1-\frac{1}{2}\eta^*)>0.
\end{equation}
Combining inequalities \eqref{eq:coef_strongly2} and \eqref{eq:coef_strongly3} with Eq.\eqref{eq:coef_strongly1},
we get
\begin{equation}
\label{eq:integral_mega}
\begin{array}{ll}
\frac{\log(1-\gamma_k)}{\log\gamma_k}
&\le \frac{\log 2-\log \eta^*}{-\log(1-\frac{1}{2}\eta^*)} \\
& = \frac{\log 2}{-\log(1-\frac{1}{2}\eta^*)} + \frac{-\log \eta^*}{-\log(1-\frac{1}{2}\eta^*)}
\end{array}
\end{equation}

For the first term on the RHS,
\begin{equation*} 
\begin{array}{ll}
-\log(1-\frac{1}{2}\eta^*)&=\log\Big(\frac{1}{1-\frac{1}{4}\sqrt{\frac{\max\{V,\log^2 T/T\}}{2DT}}}\Big)\\
& = \log\Big(1+\frac{\frac{1}{4}\sqrt{\frac{\max\{V,\log^2 T/T\}}{2DT}}}{1-\frac{1}{4}\sqrt{\frac{\max\{V,\log^2 T/T\}}{2DT}}}\Big)\\
&\ge \frac{1}{2}\frac{\frac{1}{4}\sqrt{\frac{\max\{V,\log^2 T/T\}}{2DT}}}{1-\frac{1}{4}\sqrt{\frac{\max\{V,\log^2 T/T\}}{2DT}}}\\
&\ge \frac{1}{8}\sqrt{\frac{\max\{V,\log^2 T/T\}}{2DT}}
\end{array}
\end{equation*}
where the first inequality is due to $\log(1+x)\ge\frac{1}{2}x,x\in[0,1]$ and
the second one is due to $\sqrt{\frac{\max\{V,\log^2 T/T\}}{2DT}}>0$.
As a result, 
\begin{equation*}
\begin{array}{ll}
\frac{\log 2}{-\log(1-\frac{1}{2}\eta^*)}
&\le 8 \sqrt{\frac{2DT}{\max\{V,\log^2 T/T\}}}\log 2\\
&\le 8\frac{T}{\log T}\sqrt{2D}\log 2<O(T).
\end{array}
\end{equation*}

For the second term on the RHS of Eq.\eqref{eq:integral_mega}, 
\begin{equation*}
\begin{array}{ll}
-\log \eta^* &= \log\Big(2\sqrt{\frac{2DT}{\max\{V,\log^2 T/T\}}}\Big) \\
&\le \log 2 +\frac{1}{2}\log 2D + \frac{1}{2}\log \frac{T}{\log T}
\end{array}
\end{equation*}

Combining the inequalities for $-\log \eta^*$ and $-\log(1-\frac{1}{2}\eta^*)$,
we get 
$\frac{-\log \eta^*}{-\log(1-\frac{1}{2}\eta^*)} 
\le (\log 2 +\frac{1}{2}\log 2D + \frac{1}{2}\log \frac{T}{\log T})8\frac{T}{\log T}\sqrt{2D}
\le O(T)$.

As a result, $\frac{\log(1-\gamma_k)}{\log\gamma_k}\le O(T)$ and 
$\sum\limits_{t=1}^T\frac{1}{1-\gamma_k^t} \le O(T)$.

Since using $\gamma_k$ does not increase the order when used in place of $\gamma^*$, 
we get 
\begin{equation*}
\sum\limits_{t=1}^T \Big(f_t(\theta_t^{\gamma_k})-f_t(z_t)\Big) \le O(\max\{\log T, \sqrt{TV}\})
\end{equation*}
which combining with the result of Lemma \ref{lem:meta-expert-compare} completes the proof.

\end{proof}

\paragraph{Proof of Proposition \ref{prop:lower-bound}:}

\begin{proof}

Since strongly convex problem with bounded gradient is also exp-concave 
due to Lemma \ref{lem::strongly_is_exp} shown in the next section,
we will only consider the strongly convex problem.

For the case when $V = 0$, $\cR_d$ reduces to the static regret $\cR_s$, which has the lower bound $O(\log T)$
as shown in \cite{abernethy2008optimal}.

Let us now consider the case when $V> 0$. The analysis is inspired by \cite{yang2016tracking}.
We will use $f_t(\theta) = (\theta-\epsilon_t)^2$ as the special case
to show the lower bound. Here $\epsilon_1^T$ is a sequence of independently generated
random variables from $\{-2\sigma,2\sigma\}$ with equal probabilities.
For the dynamic regret 
$\cR_d = \sum\limits_{t=1}^Tf_t(\theta_t) - \min\limits_{z_1^T\in \cS_V}\sum\limits_{t=1}^Tf_t(z_t)
\ge \sum\limits_{t=1}^Tf_t(\theta_t) - \sum\limits_{t=1}^Tf_t(z_t)$,
where $\cS_V = \{z_1^T:\sum\limits_{t=2}^T\|z_t-z_{t-1}\|\le V\}$, and
$z_t = \frac{1}{2}\epsilon_t$.
As a result, the expected value of
$\sum\limits_{t=1}^Tf_t(\theta_t) - \sum\limits_{t=1}^Tf_t(z_t)$
is $\bbE[\sum\limits_{t=1}^Tf_t(\theta_t) - \sum\limits_{t=1}^Tf_t(z_t)]$
$=$ $\bbE[\sum\limits_{t=1}^T(\theta_t^2 -2\theta_t\epsilon_t+\frac{3}{4}\epsilon_t^2)]$
$\ge$ $\sum\limits_{t=1}^T \bbE[-2\theta_t\epsilon_t +\frac{3}{4}\epsilon_t^2]$ $=$ $3\sigma^2 T$.
This implies that $\cR_d \ge 3\sigma^2 T$.
For the path length, $\sum\limits_{t=2}^T\|z_t-z_{t-1}\|\le 2\sigma T$.
Let us set $\sigma = T^{-\frac{2(1-\gamma_0)}{4-\gamma_0}}$ and $\gamma_0 \in (0,1)$.
Then $V = 2\sigma T = 2T^{\frac{2+\gamma_0}{4-\gamma_0}}$
and $(VT)^{\frac{\gamma_0}{2}}$ $=$ $2^{\frac{\gamma_0}{2}}T^{\frac{3\gamma_0}{4-\gamma_0}}$.
Then $\cR_d - \frac{3}{\sqrt{2}}(VT)^{\frac{\gamma_0}{2}}$ $\ge$ 
$3T^{\frac{3\gamma_0}{4-\gamma_0}} -  \frac{3}{\sqrt{2}}2^{\frac{\gamma_0}{2}}T^{\frac{3\gamma_0}{4-\gamma_0}}$
$\ge 0$. In other words, $\cR_d\ge O\Big((VT)^{\frac{\gamma_0}{2}}\Big)$, $\forall \gamma_0\in (0,1)$
with $V = 2T^{\frac{2+\gamma_0}{4-\gamma_0}}$.

In summary, we have
that
there always exist a 
exist a sequence of loss functions $f_1^T$
and a comparison sequence $z_1^T$
such that $\sum\limits_{t=2}^T\|z_t-z_{t-1}\|\le V = O(T^{\frac{2+\gamma_0}{4-\gamma_0}})$ and 
$\cR_d \ge \max\{O(\log T),O\big((VT)^{\frac{\gamma_0}{2}}\big)\}, \forall \gamma_0 \in (0,1)$

\end{proof}

\paragraph{Online Least-Squares Optimization}

Consider the online least-squares problem with:
\begin{equation}
\label{eq::gen_ls_loss}
f_t(\theta) = \frac{1}{2}\left\|y_t - A_t\theta\right\|^2
\end{equation}
where $A_t\in\mathbb{R}^{m\times n}$, $A_t^TA_t$ has full rank with $lI\preceq A_t^TA_t\preceq uI$, 
and $y_t\in\mathbb{R}^m$ comes from a bounded set with
$\left\|y_t\right\|\le D$.

In the main paper, we analyzed the dynamic regret of discounted recursive least squares against comparison
sequences $z_1,\ldots,z_T$ with a path length constraint
$\sum_{t=2}^T\|z_t-z_{t-1}\| \le V$. Additionally, we analyzed the trade-off between static and
dynamic regret of a gradient descent rule with comparison sequence
$\theta_t^* = \argmin_{\theta \in \cS} f_t(\theta)$. In this appendix,
we analyze the trade-off between static regret and dynamic regret with
comparison sequence $\theta_t^*$ achieved by discounted recursive
least squares. We will see that the discounted recursive least squares
achieves trade-offs depend on the condition number, $\delta = u/l$. In
particular, low dynamic regret is only guaranteed for low condition
numbers. 

Recall that discounted recursive least squares corresponds to
Algorithm~\ref{alg:discountedNewton} running with a full Newton step and $\eta
= 1$.
In this case, $P_t =
\sum\limits_{i=1}^t\gamma^{i-1}A_{t+1-i}^TA_{t+1-i} = \gamma
P_{t-1}+A_t^TA_t$, and the update rule can be written more explicitly
as 
\small
\begin{equation}
\label{eq::org_gen_ls_update}
\theta_{t+1} = \Big(\sum\limits_{i=1}^t\gamma^{i-1}A_{t+1-i}^TA_{t+1-i}\Big)^{-1}\Big(\sum\limits_{i=1}^t \gamma^{i-1}A_{t+1-i}^Ty_{t+1-i}\Big)
\end{equation}
\normalsize
The above update rule can be reformulated as:
\begin{equation}
\label{eq::gen_ls_update}
\theta_{t+1} = \theta_t - P_t^{-1}\nabla f_t(\theta_t).
\end{equation}



Before we analyze dynamic and static regret for the update \eqref{eq::gen_ls_update}, 
we first show some supporting results for $\left\|y_t-A_tx\right\|$ 
and $\left\|\nabla f_t(x)\right\|$, where $x\in \{v| v = \beta \theta_t + (1-\beta)\theta_t^*,\beta\in[0,1]\}$.

\begin{lemma}
  \label{lem:norm_gen_ls_dif}
  {\it
  Let $\theta_t$ be the result of Eq.\eqref{eq::gen_ls_update}, and $\theta_t^* = \argmin f_t(\theta)$.
  For $x\in \{v| v = \beta \theta_t + (1-\beta)\theta_t^*,\beta\in[0,1]\}$,
  If $\left\|y_t\right\|\le D$, then $\left\|y_t-A_tx\right\|\le (u/l +1)D$.
  }
\end{lemma}

\begin{proof}

$\left\|y_t-A_tx\right\| \le \left\|A_t\right\|_2\left\|x\right\|+\left\|y_t\right\|$, and 
$\left\|A_t\right\|_2 = \sqrt{\sigma_1(A_t^TA_t)}\le \sqrt{u}$. 
For $\left\|x\right\|$, we have
$\left\|x\right\| = \left\|\beta\theta_t+(1-\beta)\theta_t^*\right\|\le \beta\left\|\theta_t\right\|+(1-\beta)\left\|\theta_t^*\right\|$.

For the term $\left\|\theta_t\right\|$,
$\left\|\theta_t\right\| = \left\|\Big(\sum\limits_{i=1}^{t-1}\gamma^{i-1}A_{t-i}^TA_{t-i}\Big)^{-1}
\Big(\sum\limits_{i=1}^{t-1} \gamma^{i-1}A_{t-i}^Ty_{t-i}\Big)\right\|$,
which can be upper bounded by 
$\left\|\Big(\sum\limits_{i=1}^{t-1}\gamma^{i-1}A_{t-i}^TA_{t-i}\Big)^{-1}\right\|_2
\left\|\Big(\sum\limits_{i=1}^{t-1} \gamma^{i-1}A_{t-i}^Ty_{t-i}\Big)\right\|$.
Then we upper bound these two terms individually.

$\left\|\Big(\sum\limits_{i=1}^{t-1}\gamma^{i-1}A_{t-i}^TA_{t-i}\Big)^{-1}\right\|_2 
= \frac{1}{\sigma_n(\sum\limits_{i=1}^{t-1}\gamma^{i-1}A_{t-i}^TA_{t-i})}$. 
Since $lI\preceq A_{t-i}^TA_{t-i}\preceq uI$, 
$\frac{1-\gamma^{t-1}}{1-\gamma}lI\preceq \sum\limits_{i=1}^{t-1}\gamma^{i-1}A_{t-i}^TA_{t-i})\preceq \frac{1-\gamma^{t-1}}{1-\gamma}uI$.
Thus, $\sigma_n(\sum\limits_{i=1}^{t-1}\gamma^{i-1}A_{t-i}^TA_{t-i})\ge l\frac{1-\gamma^{t-1}}{1-\gamma}$,
which results in $\left\|\Big(\sum\limits_{i=1}^{t-1}\gamma^{i-1}A_{t-i}^TA_{t-i}\Big)^{-1}\right\|_2 
\le \frac{1-\gamma}{l(1-\gamma^{t-1})}$.

For the term $\left\|\Big(\sum\limits_{i=1}^{t-1} \gamma^{i-1}A_{t-i}^Ty_{t-i}\Big)\right\|$, 
we have $\left\|\Big(\sum\limits_{i=1}^{t-1} \gamma^{i-1}A_{t-i}^Ty_{t-i}\Big)\right\|
\le \sum\limits_{i=1}^{t-1}\gamma^{i-1}\left\|A_{t-i}^Ty_{t-i}\right\|
\le \sum\limits_{i=1}^{t-1}\gamma^{i-1}\left\|A_{t-i}^T\right\|_2\left\|y_{t-i}\right\|
\le \frac{1-\gamma^{t-1}}{1-\gamma}\sqrt{u}D$.
Then we have $\left\|\theta_t\right\|\le \frac{\sqrt{u}}{l}D$.

For $\left\|\theta_t^*\right\|$,
we have $\left\|\theta_t^*\right\| = \left\|(A_t^TA_t)^{-1}A_t^Ty_t\right\|
\le \left\|(A_t^TA_t)^{-1}\right\|_2\left\|A_t^T\right\|_2\left\|y_t\right\|
\le \frac{\sqrt{u}}{l}D$. Thus, $\left\|x\right\| \le \frac{\sqrt{u}}{l}D$
and $\left\|y_t-A_tx\right\| \le \left\|A_t\right\|_2\left\|x\right\|+\left\|y_t\right\|
\le (u/l+1)D$.

\end{proof}

\begin{corollary}
  \label{corol:norm_gen_ls_grad}
  {\it
  Let $\theta_t$ be the result of Eq.\eqref{eq::gen_ls_update} and $\theta_t^* = \argmin f_t(\theta)$.
  For $x\in \{v| v = \beta \theta_t + (1-\beta)\theta_t^*,\beta\in[0,1]\}$, we have 
  $\left\|\nabla f_t(x)\right\| \le \sqrt{u}(u/l+1)D$.
  }
\end{corollary}

\begin{proof}

For $\left\|\nabla f_t(x)\right\|$, we have $\left\|\nabla f_t(x)\right\| = \left\|A_t^TA_tx-A_t^Ty_t\right\|
\le \left\|A_t^T\right\|_2\left\|A_tx-y_t\right\|\le \sqrt{u}(u/l+1)D$,
where the second inequality is due to Lemma \ref{lem:norm_gen_ls_dif} and the assumption of $A_t^TA_t\preceq uI$.

\end{proof}

Moreover, we need to obtain the relationship between $\theta_{t+1}-\theta_t^*$ and $\theta_t-\theta_t^*$ 
as another necessary step to get the dynamic regret.

\begin{lemma}
  \label{lem:gen_ls_var_path}
  {\it
  Let $\theta_t^*$ be the solution to $f_t(\theta)$ in Eq.\eqref{eq::gen_ls_loss}.
  When we use the discounted recursive least-squares update in Eq.\eqref{eq::gen_ls_update},
  the following relationship is obtained:
  \begin{equation*}
  \begin{array}{ll}
  \theta_{t+1} -\theta_t^*\\
   = \big(I-\gamma^{-1}P_{t-1}^{-1}A_t^T(I+A_t\gamma^{-1}P_{t-1}^{-1}A_t^T)^{-1}A_t\big)(\theta_t-\theta_t^*) \\
  = \Big(I+\gamma^{-1}P_{t-1}^{-1}A_t^TA_t\Big)^{-1}(\theta_t-\theta_t^*) 
  \end{array}
  \end{equation*}
  }
\end{lemma}

\begin{proof}

If we set $\Phi_t = \sum\limits_{i=1}^t\gamma^{i-1}A_{t+1-i}^Ty_{t+1-i} = \gamma \Phi_{t-1}+A_t^Ty_t$,
then according to the update of $\theta_{t+1}$ in Eq.\eqref{eq::org_gen_ls_update}, 
we have $\theta_{t+1} = (A_t^TA_t+\gamma P_{t-1})^{-1}(A_t^Ty_t+\gamma\Phi_{t-1})$,
which by the use of inverse lemma can be further reformulated as:
\begin{equation}
\begin{array}{ll}
\theta_{t+1} &= \Big(\gamma^{-1}P_{t-1}^{-1}- \gamma^{-2}P_{t-1}^{-1}A_t^T(I+\\
&\quad A_t\gamma^{-1}P_{t-1}^{-1}A_t^T)^{-1}A_tP_{t-1}^{-1}\Big)\big(A_t^Ty_t+\gamma\Phi_{t-1}\big) \\
\end{array}
\end{equation}
Then for $\theta_{t+1}-\theta_t^* = \theta_{t+1}-(A_t^TA_t)^{-1}A_t^Ty_t$, we have:
\scriptsize
\begin{equation}
\begin{array}{ll}
\theta_{t+1}-\theta_t^* \\
= \underbrace{\big(I-\gamma^{-1}P_{t-1}^{-1}A_t^T(I+A_t\gamma^{-1}P_{t-1}^{-1}A_t^T)^{-1}A_t\big)}_{\circled{1}}\theta_t
+\underbrace{\gamma^{-1}P_{t-1}^{-1}A_t^Ty_t}_{\circled{2.1}} \\
-\underbrace{\big(\gamma^{-2}P_{t-1}^{-1}A_t^T(I+A_t\gamma^{-1}P_{t-1}^{-1}A_t^T)^{-1}A_tP_{t-1}^{-1}
-(A_t^TA_t)^{-1}\big)A_t^Ty_t}_{\circled{2.2}}
\end{array}
\end{equation}
\normalsize
We want to prove $\circled{2.1}+\circled{2.2} = \circled{1}(-\theta_t^*) = \circled{1}(-(A_t^TA_t)^{-1}A_t^Ty_t) = \circled{3}$.

Since $A(I+BA)^{-1}B = AB(I+AB)^{-1}=(I+AB)^{-1}AB$, for any compatible matrix $A$ and $B$, we have:
\scriptsize
\begin{equation}
\begin{array}{ll}
\circled{3}\\
= -\big[I-\gamma^{-1}P_{t-1}^{-1}A_t^T(I+A_t\gamma^{-1}P_{t-1}^{-1}A_t^T)^{-1}A_t\big](A_t^TA_t)^{-1}A_t^Ty_t \\
= -\big[I-(I+\gamma^{-1}P_{t-1}^{-1}A_t^TA_t)^{-1}\gamma^{-1}P_{t-1}^{-1}A_t^TA_t\big](A_t^TA_t)^{-1}A_t^Ty_t \\
= -\big[(A_t^TA_t)^{-1} - (I+\gamma^{-1}P_{t-1}^{-1}A_t^TA_t)^{-1}\gamma^{-1}P_{t-1}^{-1}\big]A_t^Ty_t
\end{array}
\end{equation}
\normalsize
Also, for any compatible $P$, we have $(I+P)^{-1} = I-(I+P)^{-1}P$. 
Then $(I+\gamma^{-1}P_{t-1}^{-1}A_t^TA_t)^{-1} = I - (I+\gamma^{-1}P_{t-1}^{-1}A_t^TA_t)^{-1}\gamma^{-1}P_{t-1}^{-1}A_t^TA_t$.
Then $\circled{3} = -\big[(A_t^TA_t)^{-1}-\gamma^{-1}P_{t-1}^{-1}
+(I+\gamma^{-1}P_{t-1}^{-1}A_t^TA_t)^{-1}\gamma^{-2}P_{t-1}^{-1}A_t^TA_tP_{t-1}^{-1}\big]A_t^Ty_t$.
Compared with $\circled{2.1}+\circled{2.2}$, we are left to prove
$(I+\gamma^{-1}P_{t-1}^{-1}A_t^TA_t)^{-1}\gamma^{-2}P_{t-1}^{-1}A_t^TA_tP_{t-1}^{-1}
 = \gamma^{-2}P_{t-1}^{-1}A_t^T(I+A_t\gamma^{-1}P_{t-1}^{-1}A_t^T)^{-1}A_tP_{t-1}^{-1}$, which is always true.

As a result, we have $\theta_{t+1} -\theta_t^* 
= \big(I-\gamma^{-1}P_{t-1}^{-1}A_t^T(I+A_t\gamma^{-1}P_{t-1}^{-1}A_t^T)^{-1}A_t\big)(\theta_t-\theta_t^*)$,
which can be simplified as $\theta_{t+1} -\theta_t^* 
=\big(I+\gamma^{-1}P_{t-1}^{-1}A_t^TA_t\big)^{-1}(\theta_t-\theta_t^*)$.

\end{proof}

\begin{corollary}
  \label{corol:gen_ls_var_path_norm}
  {\it
  Let $\theta_t^*$ be the solution to $f_t(\theta)$ in Eq.\eqref{eq::gen_ls_loss}.
  When we use the discounted recursive least-squares update in Eq.\eqref{eq::gen_ls_update},
  the following relation is obtained:
  \begin{equation*}
  \begin{array}{ll}
  \left\|\theta_{t+1} -\theta_t^*\right\| & \le \sqrt{\frac{u}{l}}\frac{u\gamma}{u\gamma+l(1-\gamma)}\left\|\theta_t-\theta_t^*\right\|
  \end{array}
  \end{equation*}
  }
\end{corollary}

\begin{proof}

From Lemma \ref{lem:gen_ls_var_path} we know that
\begin{equation*}
\theta_{t+1} -\theta_t^* = \Big(I+\gamma^{-1}P_{t-1}^{-1}A_t^TA_t\Big)^{-1}(\theta_t-\theta_t^*) 
\end{equation*}
which can be reformulated as:
\footnotesize
\begin{equation*}
\theta_{t+1} -\theta_t^* = P_{t-1}^{-1/2}(I+\gamma^{-1}P_{t-1}^{-1/2}A_t^TA_tP_{t-1}^{-1/2})^{-1}P_{t-1}^{1/2}(\theta_t-\theta_t^*) 
\end{equation*}
\normalsize
which gives us the following inequality:
\begin{equation*}
\begin{array}{l}
\left\|\theta_{t+1} -\theta_t^*\right\| \\
\le \left\|P_{t-1}^{-1/2}\right\|_2\left\|(I+\gamma^{-1}P_{t-1}^{-1/2}A_t^TA_tP_{t-1}^{-1/2})^{-1}\right\|_2\\
\quad\left\|P_{t-1}^{1/2}\right\|_2\left\|\theta_t-\theta_t^*\right\|
\end{array}
\end{equation*}
Then we will upper bound the terms on the right-hand side individually.

Since $lI\preceq A_{t-i}^TA_{t-i}\preceq uI$, 
$\frac{1-\gamma^{t-1}}{1-\gamma}lI\preceq P_{t-1}= \sum\limits_{i=1}^{t-1}\gamma^{i-1}A_{t-i}^TA_{t-i}
\preceq \frac{1-\gamma^{t-1}}{1-\gamma}uI$.

For the term $\left\|P_{t-1}^{-1/2}\right\|_2$, 
we have $\left\|P_{t-1}^{-1/2}\right\|_2 = \frac{1}{\sqrt{\sigma_n(P_{t-1})}}$. 
Since $\sigma_n(P_{t-1})\ge \frac{1-\gamma^{t-1}}{1-\gamma}l$, 
$\left\|P_{t-1}^{-1/2}\right\|_2\le \frac{1}{\sqrt{l}}\sqrt{\frac{1-\gamma}{1-\gamma^{t-1}}}$.

For the term $\left\|P_{t-1}^{1/2}\right\|_2$, 
we have $\left\|P_{t-1}^{1/2}\right\|_2$ $=$ $\sqrt{\sigma_1(P_{t-1})}$.
Since $\sigma_1(P_{t-1})\le \frac{1-\gamma^{t-1}}{1-\gamma}u$,
$\left\|P_{t-1}^{1/2}\right\|_2\le \sqrt{u}\sqrt{\frac{1-\gamma^{t-1}}{1-\gamma}}$.

For the term $\left\|(I+\gamma^{-1}P_{t-1}^{-1/2}A_t^TA_tP_{t-1}^{-1/2})^{-1}\right\|_2$,
we have $\left\|(I+\gamma^{-1}P_{t-1}^{-1/2}A_t^TA_tP_{t-1}^{-1/2})^{-1}\right\|_2
= 1/\sigma_n(I$$+$$\gamma^{-1}P_{t-1}^{-1/2}A_t^TA_tP_{t-1}^{-1/2})$.
For the term
$\sigma_n(I$$+$$\gamma^{-1}P_{t-1}^{-1/2}A_t^TA_tP_{t-1}^{-1/2})$,
it is equal to 
$1+\sigma_n(\gamma^{-1}P_{t-1}^{-1/2}A_t^TA_tP_{t-1}^{-1/2})$,
which is lower bounded by
$1+\gamma^{-1}\sigma_n(P_{t-1}^{-1/2})\sigma_n(A_t^TA_t)\sigma_n(P_{t-1}^{-1/2})$.

Since $\sigma_n(P_{t-1}^{-1/2}) = \frac{1}{\sqrt{\sigma_1(P_{t-1})}}$
and $\sigma_1(P_{t-1})\le \frac{1-\gamma^{t-1}}{1-\gamma}u$, we have
$\sigma_n(P_{t-1}^{-1/2})\ge \frac{1}{\sqrt{u}}\sqrt{\frac{1-\gamma}{1-\gamma^{t-1}}}$.
Together with $\sigma_n(A_t^TA_t)\ge l$, we have
$\sigma_n(P_{t-1}^{-1/2}A_t^TA_tP_{t-1}^{-1/2}) 
\ge \frac{l}{u}\frac{1-\gamma}{1-\gamma^{t-1}}$,
which results in
$\left\|(I+\gamma^{-1}P_{t-1}^{-1/2}A_t^TA_tP_{t-1}^{-1/2})^{-1}\right\|_2
\le \frac{1}{1+\gamma^{-1}\frac{l}{u}\frac{1-\gamma}{1-\gamma^{t-1}}}$.

Combining the above three terms' inequalities, we have
$\left\|\theta_{t+1} -\theta_t^*\right\| 
\le \sqrt{\frac{u}{l}}\frac{u(\gamma-\gamma^t)}{u(\gamma-\gamma^t)+l(1-\gamma)}\left\|\theta_t-\theta_t^*\right\|
\le \sqrt{\frac{u}{l}}\frac{u\gamma}{u\gamma+l(1-\gamma)}\left\|\theta_t-\theta_t^*\right\|$.

\end{proof}

Now we are ready to present the dynamic regret for the general recursive least-squares update:
\begin{theorem}
\label{thm::gen_ls_dynamic_regret}
{\it 
Let $\theta_t^*$ be the solution to $f_t(\theta)$ in Eq.\eqref{eq::gen_ls_loss}
and $\delta = u/l\ge 1$ be the condition number.
When using the discounted recursive least-squares update in Eq.\eqref{eq::gen_ls_update} 
with $\gamma <\frac{1}{\delta^{3/2}-\delta+1}$
and $\rho = \sqrt{\frac{u}{l}}\frac{u\gamma}{u\gamma+l(1-\gamma)}<1$,
  we can upper bound the dynamic regret:
  \small
  \begin{equation*}
  \mathcal{R}_d \le \sqrt{u}(u/l+1)D\frac{1}{1-\rho}\big(\left\|\theta_1-\theta_1^*\right\| +
+ \sum\limits_{t=2}^T\left\|\theta_t^*-\theta_{t-1}^*\right\|\big) 
  \end{equation*}
  \normalsize
}
\end{theorem}

\begin{proof}
The proof follows the similar steps in the proof of Theorem \ref{thm::quad_dynamic_regret}.
First, we use the Mean Value Theorem to get 
$f_t(\theta_t)-f_t(\theta_t^*) = \nabla f_t(x)^T(\theta_t-\theta_t^*)
\le \left\|\nabla f_t(x)\right\|\left\|\theta_t-\theta_t^*\right\|$,
where $x\in \{v| v = \beta \theta_t + (1-\beta)\theta_t^*,\beta\in[0,1]\}$.
According to Corollary \ref{corol:norm_gen_ls_grad},
$\left\|\nabla f_t(x)\right\|\le \sqrt{u}(u/l+1)D$.
As a result, 
$\sum\limits_{t=1}^T\big(f_t(\theta_t)-f_t(\theta_t^*)\big)\le \sqrt{u}(u/l+1)D\sum\limits_{t=1}^T\left\|\theta_t-\theta_t^*\right\|$.

Now we need to upper bound the term $\sum\limits_{t=1}^T\left\|\theta_t-\theta_t^*\right\|$.
$\sum\limits_{t=1}^T\left\|\theta_t-\theta_t^*\right\| = \left\|\theta_1-\theta_1^*\right\| 
+ \sum\limits_{t=2}^T\left\|\theta_t-\theta_{t-1}^*+\theta_{t-1}^*-\theta_t^*\right\|
\le \left\|\theta_1-\theta_1^*\right\| + \sum\limits_{t=1}^{T-1}\left\|\theta_{t+1}-\theta_{t}^*\right\| 
+ \sum\limits_{t=2}^T\left\|\theta_t^*-\theta_{t-1}^*\right\| 
\le \left\|\theta_1-\theta_1^*\right\| + \sum\limits_{t=1}^{T}\left\|\theta_{t+1}-\theta_{t}^*\right\| 
+ \sum\limits_{t=2}^T\left\|\theta_t^*-\theta_{t-1}^*\right\|$.
According to Corollary \ref{corol:gen_ls_var_path_norm}, 
$\left\|\theta_{t+1} -\theta_t^*\right\| \le \rho\left\|\theta_t-\theta_t^*\right\|$.
$\sum\limits_{t=1}^T\left\|\theta_t-\theta_t^*\right\| \le
\left\|\theta_1-\theta_1^*\right\| + \rho\sum\limits_{t=1}^{T}\left\|\theta_{t}-\theta_{t}^*\right\| 
+ \sum\limits_{t=2}^T\left\|\theta_t^*-\theta_{t-1}^*\right\|$,
which can be reformulated as 
$\sum\limits_{t=1}^T\left\|\theta_t-\theta_t^*\right\| \le
\frac{1}{1-\rho}(\left\|\theta_1-\theta_1^*\right\| +
+ \sum\limits_{t=2}^T\left\|\theta_t^*-\theta_{t-1}^*\right\|)$.
Then $\mathcal{R}_d = \sum\limits_{t=1}^T\big(f_t(\theta_t)-f_t(\theta_t^*)\big)
\le \sqrt{u}(u/l+1)D\frac{1}{1-\rho}(\left\|\theta_1-\theta_1^*\right\| +
+ \sum\limits_{t=2}^T\left\|\theta_t^*-\theta_{t-1}^*\right\|)$.

\end{proof}

In the above Theorem \ref{thm::gen_ls_dynamic_regret}, the valid range of $\gamma$ is in $(0,1/(\delta^{3/2}-\delta+1))$.
Let us now examine the requirement of $\gamma$ to achieve the sub-linear static regret:
\begin{theorem}
\label{thm::gen_ls_static_regret}
{\it Let $\theta^*$ be the solution to $\min\sum\limits_{t=1}^T f_t(\theta)$. 
  When using the discounted recursive least-squares update in Eq.\eqref{eq::gen_ls_update} with $1-\gamma = 1/T^{\alpha}, \alpha\in (0,1)$,
  we can upper bound the static regret:
  \begin{equation*}
  \mathcal{R}_s \le O(T^{1-\alpha})
  \end{equation*}
}
\end{theorem}

\begin{proof}

The proof follows the analysis of the online Newton method \cite{hazan2007logarithmic}.
From the update in Eq.\eqref{eq::gen_ls_update}, 
we have $\theta_{t+1}-\theta^* = \theta_t-\theta^*-P_{t}^{-1}\nabla f_t(\theta_t)$
and $P_t(\theta_{t+1}-\theta^*) = P_t(\theta_t-\theta^*)-\nabla f_t(\theta_t)$.
Multiplying the two equalities, we have
$(\theta_{t+1}-\theta^*)^TP_t(\theta_{t+1}-\theta^*) 
= (\theta_t-\theta^*)^TP_t(\theta_t-\theta^*)-2\nabla f_t(\theta_t)^T(\theta_t-\theta^*)
+ \nabla f_t(\theta_t)^TP_{t}^{-1}\nabla f_t(\theta_t)$.

After the reformulation, we have 
$\nabla f_t(\theta_t)^T(\theta_t-\theta^*) = \frac{1}{2}\nabla f_t(\theta_t)^TP_{t}^{-1}\nabla f_t(\theta_t)
+\frac{1}{2}(\theta_t-\theta^*)^TP_t(\theta_t-\theta^*)-\frac{1}{2}(\theta_{t+1}-\theta^*)^TP_t(\theta_{t+1}-\theta^*)
\le \frac{1}{2}\nabla f_t(\theta_t)^TP_{t}^{-1}\nabla f_t(\theta_t)
+\frac{1}{2}(\theta_t-\theta^*)^TP_t(\theta_t-\theta^*)-\frac{1}{2}(\theta_{t+1}-\theta^*)^T\gamma P_t(\theta_{t+1}-\theta^*)$.

Summing the above inequality from $t=1$ to $T$, we have:
$\sum\limits_{t=1}^T \nabla f_t(\theta_t)^T(\theta_t-\theta^*) 
\le \sum\limits_{t=1}^T\frac{1}{2}\nabla f_t(\theta_t)^TP_{t}^{-1}\nabla f_t(\theta_t)
+ \frac{1}{2}(\theta_1-\theta^*)^TP_1(\theta_1-\theta^*)
+ \sum\limits_{t=2}^T\frac{1}{2}(\theta_t-\theta^*)^T(P_t-\gamma P_{t-1})(\theta_t-\theta^*)
- \frac{1}{2}(\theta_{T+1}-\theta^*)^T\gamma P_T(\theta_{T+1}-\theta^*)
\le \sum\limits_{t=1}^T\frac{1}{2}\nabla f_t(\theta_t)^TP_{t}^{-1}\nabla f_t(\theta_t)
+ \frac{1}{2}(\theta_1-\theta^*)^T(P_1-A_1^TA_1)(\theta_1-\theta^*)
+ \sum\limits_{t=1}^T\frac{1}{2}(\theta_t-\theta^*)^TA_t^TA_t(\theta_t-\theta^*)$.

Since $P_1 = A_1^TA_1$ and $f_t(\theta_t) - f_t(\theta^*) = \nabla f_t(\theta_t)^T(\theta_t-\theta^*)
-\frac{1}{2}(\theta_t-\theta^*)^TA_t^TA_t(\theta_t-\theta^*)$,
we reformulate the above inequality as:
\begin{equation}
\begin{array}{ll}
\sum\limits_{t=1}^T \Big(f_t(\theta_t) - f_t(\theta^*) \Big)
\\= \sum\limits_{t=1}^T \Big(\nabla f_t(\theta_t)^T(\theta_t-\theta^*)
- \frac{1}{2}(\theta_t-\theta^*)^TA_t^TA_t(\theta_t-\theta^*)\Big) \\
\le \sum\limits_{t=1}^T\frac{1}{2}\nabla f_t(\theta_t)^TP_{t}^{-1}\nabla f_t(\theta_t)\\
 = \sum\limits_{t=1}^T\frac{1}{2}(A_t\theta_t-y_t)^TA_tP_t^{-1}A_t^T(A_t\theta_t-y_t) \\
 \le \sum\limits_{t=1}^T \frac{1}{2}\sigma_1(P_t^{-1/2}A_t^TA_tP_t^{-1/2})\left\|A_t\theta_t-y_t\right\|^2
\end{array}
\end{equation}
Since $\sigma_1(P_t^{-1/2}A_t^TA_tP_t^{-1/2})\le \sigma_1(P_t^{-1})\sigma_1(A_t^TA_t)
= \frac{1}{\sigma_n(P_t)}\sigma_1(A_t^TA_t)$.
From the proof of Corollary \ref{corol:gen_ls_var_path_norm} we know that 
$\sigma_n(P_t)\ge \frac{1-\gamma^t}{1-\gamma}l$ and $\sigma_1(A_t^TA_t)\le u$.
Then $\sigma_1(P_t^{-1/2}A_t^TA_tP_t^{-1/2})\le \frac{u}{l}\frac{1-\gamma}{1-\gamma^t}$.
As a result, we have 
\begin{equation}
\begin{array}{ll}
\sum\limits_{t=1}^T \Big(f_t(\theta_t) - f_t(\theta^*) \Big) 
&\le \sum\limits_{t=1}^T\frac{1}{2}\frac{u}{l}\frac{1-\gamma}{1-\gamma^t}\left\|A_t\theta_t-y_t\right\|^2 \\
&\le \sum\limits_{t=1}^T\frac{1}{2}\frac{u}{l}\frac{1-\gamma}{1-\gamma^t}(u/l+1)^2D^2 \\
&\le O(T^{1-\alpha})
\end{array}
\end{equation}
where the second inequality is due to Lemma \ref{lem:norm_gen_ls_dif} 
and the third inequality is due to the fact that $\sum\limits_{t=1}^T1/(1-\gamma^t)\le O(T)$
as shown in the proof of Theorem \ref{thm::quad_static_regret}.

\end{proof}

Recall that the valid range of $\gamma$ in Theorem \ref{thm::gen_ls_dynamic_regret} is
$(0,1/(\delta^{3/2}-\delta+1))$, 
while having sub-linear static regret requires $\gamma = \frac{T^{\alpha}-1}{T^{\alpha}}$.
Although for some specific $T$, there might be some intersection.
In general, these two are contradictory.
However, as discussed in the main body of the paper, more flexible
trade-offs between static and dynamic regret can be achieved via the
gradient descent rule.

\end{document}